\documentclass{article}

 \usepackage[preprint]{neurips_2026}

\usepackage[utf8]{inputenc} 
\usepackage[T1]{fontenc}    
\usepackage{hyperref}       
\usepackage{url}            
\usepackage{booktabs}       
\usepackage{amsfonts}       
\usepackage{nicefrac}       
\usepackage{microtype}      
\usepackage{xcolor}         

\title{Online Allocation with Unknown Shared Supply}

\author{
Tzeh Yuan Neoh\thanks{Equal contribution}\\
Harvard University\\
\texttt{tzehyuan\_neoh@g.harvard.edu}
\And
Davin Choo\footnotemark[1]\\
Harvard University\\
\texttt{davinchoo@seas.harvard.edu}
\And
Mengchu Yue\\
Harvard University\\
\texttt{mengchuyue@g.harvard.edu}
\And
Milind Tambe\\
Harvard University\\
\texttt{tambe@seas.harvard.edu}
}

\usepackage{algorithm}
\usepackage{algorithmicx}
\usepackage{algpseudocodex}
\algrenewcommand\algorithmicrequire{\textbf{Input:}}
\algrenewcommand\algorithmicensure{\textbf{Output:}}

\usepackage{multirow}
\usepackage{mathtools}
\usepackage{subcaption}
\usepackage{enumerate}
\usepackage{bbm}
\usepackage{amsmath}
\usepackage{amssymb}
\usepackage{amsthm}
\usepackage{pifont}
\usepackage{booktabs}
\usepackage{xcolor}
\definecolor{darkgreen}{rgb}{0,0.5,0}
\usepackage{hyperref}
\hypersetup{
    unicode=false,          
    colorlinks=true,        
    linkcolor=red,          
    citecolor=darkgreen,    
    filecolor=magenta,      
    urlcolor=blue           
}
\usepackage[capitalize, nameinlink]{cleveref}
\makeatletter
\AddToHook{cmd/appendix/before}{\def\cref@section@alias{appendix}}
\makeatother

\usepackage{thm-restate}
\theoremstyle{plain}
\newtheorem{theorem}{Theorem}

\newtheorem{lemma}[theorem]{Lemma}

\theoremstyle{definition}
\newtheorem{definition}[theorem]{Definition}

\theoremstyle{remark}

\usepackage{pgfplots}
\usepgfplotslibrary{fillbetween}
\usepackage{tikz}
\usetikzlibrary{calc, graphs, graphs.standard, shapes, arrows, arrows.meta, positioning, decorations.pathreplacing, decorations.markings, decorations.pathmorphing, fit, matrix, patterns, shapes.misc, tikzmark}

\newcommand{\OSSA}{\texttt{OSSA}}
\newcommand{\GPA}{\texttt{GPA}}
\newcommand{\LAGPA}{\texttt{LA-GPA}}
\newcommand{\cost}{\mathrm{cost}}
\newcommand{\OPT}{\textsc{Opt}}
\newcommand{\ALG}{\textsc{Alg}}

\newcommand{\eps}{\varepsilon}
\newcommand{\E}{\mathbb{E}}
\newcommand{\N}{\mathbb{N}}

\begin{document}

\maketitle

\begin{abstract}
Many real-world resource allocation systems, such as humanitarian logistics and vaccine distribution, must preposition limited supply across multiple locations \emph{before} demand is realized while stockouts incur irreversible service losses.
To study this, we introduce the \emph{Online Shared Supply Allocation} (\OSSA{}) problem, a stateful online model in which a central hub allocates a finite, unknown supply to multiple sites facing sequential demand under fixed-charge transportation costs and lost-sales penalties.
Unlike classical make-to-stock or make-to-order inventory models, \OSSA{} precludes backlogging and replenishment only hedges against \emph{future} demand.
To tackle \OSSA{}, we propose a deterministic threshold-proportional policy \GPA{} and prove that it achieves a $4/3$-approximation to the offline optimum up to an additive term independent of the total supply.
We complement this with matching lower bounds showing that the $4/3$ ratio is tight and that the additive-error dependence is unavoidable, even for randomized algorithms that know the total supply upfront.
Finally, we develop a learning-augmented extension to \GPA{} that principally incorporates imperfect forecasts (e.g., from human experts or ML models) commonly available in practice, enabling us to exploit high-quality advice while being robust against arbitrary bad ones.
Synthetic and real-world experiments show that \GPA{} outperforms natural baselines with global supply is scarce.
\end{abstract}

\section{Introduction}

Many real-world systems require allocating a limited stock of resources across multiple locations \emph{before} demand is realized.
Examples arise in immunization supply chains, where vaccines must be prepositioned at service-delivery sites to ensure coverage \cite{rao2017immunization}; in humanitarian logistics, where relief supplies are staged at distribution centers in anticipation of uncertain demand \cite{balcik2008facility}; and in maintenance systems, where spare parts are stocked locally to respond to stochastic failures \cite{sherbrooke2004optimal}.
Furthermore, there may be upstream supply may be interrupted unexpectedly; for example, a humanitarian hub may have an unpredictable influx of donations or face funding cuts \cite{tezuka2026impact,betterworldcampaign}.
In these settings, insufficient local inventory results in immediate service loss --- such as missed vaccinations or food aid stockouts --- which is naturally modeled using lost-sales formulations \cite{zipkin2008old}, where unmet demand incurs an irreversible penalty rather than being backlogged.

We formalize these challenges through the \emph{Online Shared Supply Allocation (\OSSA{})} problem (\cref{def:ossa}), a stylized online model in which a central hub allocates supply over time to multiple sites facing sequential demand.
Each site maintains local inventory, and demand are either satisfied from existing stock on-hand stock upon arrival, or immediately incur an irreversible penalty.
There is no backlogging and stock replenishment (incurring fixed, capacity-constrained transportation costs) affects only \emph{future} demand.
To model supply uncertainty in the hub, the online algorithm only learns that the global supply $s$ has been exhausted when it receives less resupply than requested.\footnote{From an algorithmic perspective, knowing total supply $s$ \emph{alone} upfront does not help: an unknown amount of demand may be ``prepended'' to the original demand sequence, reducing the problem to unknown $s$ setting. However, if both $s$ and sitewise total demand is known, then one can efficiently compute the optimal solution; see \cref{sec:online}.}

\OSSA{} is characterized by three interacting features that couple decisions across both space and time.
First, a \emph{shared resource} induces multi-site coupling, as allocating supply to one location reduces availability elsewhere.
Second, the system exhibits \emph{stateful inventory dynamics}, where allocations act as prepositioned stock that persists over time to hedge against future demand.
Third, the planner operates under joint uncertainty over both demand arrivals and the total resource budget, the latter of which is revealed only upon exhaustion.
Consequently, \OSSA{} is fundamentally a problem of \emph{robust rationing}: the planner must balance immediate sitewise penalties due to unmet demand against the risk of prematurely exhausting a shared, unknown global resource.

\textbf{Relation to prior models.}
\OSSA{} sits at the intersection of multi-echelon inventory theory and online resource allocation, but differs fundamentally from both.
Unlike classical inventory models, such as the one-warehouse multi-retailer problem \cite{arkin1989computational,roundy198598}, \OSSA{} operates in a lost-sales regime where demand must be served immediately and cannot be backlogged.
Moreover, supply is finite, non-replenishable, and \emph{unknown}, necessitating robust rationing under adversarial demand.
Meanwhile, in contrast to online allocation models such as Adwords \cite{mehta2007adwords,devanur2009adwords,mehta2010online}, \OSSA{} exhibits complex temporal coupling between current rationing and future lost-sales penalties over time.
Moreover, the presence of fixed-charge transportation costs induces a non-convex objective, precluding the use of standard convex or primal-dual techniques
See \cref{sec:related} for further comparison.

\textbf{Our contributions.}
We provide a comprehensive theoretical study of the \OSSA{} problem, establishing tight guarantees for online allocation under stateful inventory and shared, uncertain supply.

\begin{itemize}
    \item \textbf{Optimal online algorithm for \OSSA{}:}
    We propose a deterministic online allocation policy based on threshold-proportional balancing called \GPA{}.
    We prove that \GPA{} achieves a $4/3$-approximation to the optimal offline solution, up to an additive term that depends only on local sitewise parameters and is independent of the total global supply $s$.

    \item \textbf{Matching lower bounds:} 
    We establish the optimality of our algorithmic result through two hardness theorems using Yao's minimax principle.
    More precisely, we prove that:
    (i) the $4/3$ ratio is the best possible for any randomized algorithm, and improving this ratio necessitates an additive error that scales with global supply $s$; and
    (ii) any asymptotic improvement in the additive term results in a superconstant approximation ratio.
    Notably, these bounds hold even if the global supply $s$ is known to the algorithm in advance.

    \item \textbf{Learning-Augmented \OSSA{}:}
    To bridge the gap between adversarial theory and practical forecasting, we extend our framework to the learning-augmented setting \cite{lykouris2021competitive,mitzenmacher2022algorithms}, where decision makers may exploit imperfect predictions on supply and demand forecasts.
    With respect to a distrust hyperparameter $\lambda$, we prove that this algorithm maintains \emph{consistency} (matching the offline optimum under perfect predictions for low $\lambda$) and \emph{smoothness} (performance degrades gracefully with prediction error), while ensuring \emph{robustness} (worst-case guarantees even under adversarial predictions for high $\lambda$).

    \item \textbf{Empirical Evaluation:}
    We evaluate our algorithms on synthetic and real-world inspired datasets.
    Our experiments demonstrate that \GPA{} significantly outperforms standard inventory heuristics and that the learning-augmented extension can effectively leverage forecasts to reduce total costs while remaining resilient to high prediction error.
\end{itemize}

Notably, our algorithms are deterministic and assume no prior knowledge of $s$ while our hardness results apply even for randomized that know $s$ in advance.
Full proofs are deferred to the appendix.

\section{Model}
\label{sec:model}

\textbf{Notation.}
For any $x \in \mathbb{R}$, let $(x)_+ = \max\{0, x\}$, and let $[n] = \{1, \dots, n\}$.

\begin{definition}[The Online Shared Supply Allocation (\OSSA{}) Problem]
\label{def:ossa}
An instance of the \OSSA{} problem is defined by a central hub with a fixed, unknown global supply $s \in \mathbb{N}$, and a set of $n$ sites. 
The fixed global supply $s$ is \emph{unknown} to the online algorithm and is revealed only upon exhaustion, and each unit of unmet demand incurs penalty $p \geq 0$.
Each site $i \in [n]$ is characterized by a tuple $(w_i, c_i, b_i)$.
Replenishment is governed by a fixed-charge transportation cost: each shipment to site $i$ has a capacity of $c_i \geq 1$ units incurs a cost $w_i \geq 0$, regardless of utilization.
Finally, $b_i \in \mathbb{N}$ denotes a known upper bound on the demand arrival at site $i$ in any single time step.
The goal is to choose an online allocation sequence $\bar{\ell} = \{ \ell_i^{t} \}$ to minimize the total cost $\cost(\bar{\ell})$ defined below:
\begin{equation}
\label{eq:main-objective}
\cost(\bar{\ell})
= \sum_{i \in [n]} \sum_{t \geq 1} \left( w_i \left\lceil \frac{\ell_i^t}{c_i} \right\rceil + p (d_i^t - k_i^t)_+ \right)
\end{equation}
where $d_i^t$ and $k_i^t$ are the demand and available stock at site $i \in [n]$ for time step $t \geq 1$.
\end{definition}

\textbf{Online order of events.}
At each discrete time step $t \geq 1$, the following sequence occurs at site $i$:
\begin{enumerate}
    \item \textbf{Demand realization:}
    Demand $d_i^t \in \mathbb{N}$ arrives and is served using current stock $k_i^t$.
    A penalty $p(d_i^t - k_i^t)_+$ is incurred for any deficit, and the remaining inventory is $r_i^t = (k_i^t - d_i^t)_+$.
    
    \item \textbf{Replenishment:}
    After observing $\{d_i^t\}_{i \in [n]}$, the online algorithm \emph{tries} to resupply $\tilde{\ell}_i^t$ for each site and receives $\ell_i^t \leq \tilde{\ell}_i^t$ from the hub.
    If $\ell_i^t < \tilde{\ell}_i^t$, we learn that the supply is exhausted.
    
    \item \textbf{Inventory update:}
    Local stock for the next step becomes $k_i^{t+1} = r_i^t + \ell_i^t$.
\end{enumerate}

We evaluate an online algorithm $\ALG$ against an optimal offline benchmark $\OPT$ that observes $s$ and all demand $\{d_i^t\}_{i \in [n], t \geq 1}$ a priori.
We seek guarantees of the form $\cost(\ALG) \leq \alpha \cdot \cost(\OPT) + \beta$, where $\alpha \geq 1$ is the competitive ratio and $\beta$ is an additive error independent of the total supply $s$.

\textbf{Simplifying model assumptions.}
Fix a site $i \in [n]$.
Without loss of generality, we assume $k_i^1 = b_i$ and $w_i \leq p c_i$.
The former ensures the algorithm is not penalized for first-step demand before it can react (setting $k_i^1 = 0$ merely shifts the additive constant $\beta$ by at most $\sum_{i=1}^n p b_i$), while the latter ensures that replenishment does not worsen the objective.
By relabeling, we may also assume that sites are ordered by their fractional transportation cost such that $\frac{w_1}{c_1} \leq \frac{w_2}{c_2} \leq \dots \leq \frac{w_n}{c_n}$.

\section{Related Work}
\label{sec:related}

The \OSSA{} problem is related to several classical lines of work, but differs from each in ways that are central to its algorithmic difficulty.
At a high level, \OSSA{} combines (i) multi-site inventory coordination, (ii) lost-sales dynamics, (iii) online uncertainty about future demand, and (iv) a shared finite stock whose total amount is unknown in advance.

\paragraph{Multi-echelon inventory and joint replenishment.}

\OSSA{} is structurally related to classical multi-echelon inventory models, particularly the one-warehouse multi-retailer (\texttt{OWMR}) problem \cite{roundy198598,arkin1989computational} and its special case, the joint replenishment problem (\texttt{JRP}) \cite{federgruen1992joint,levi2008constant}.
These models typically consider a central supplier serving multiple sites, aiming to minimize shared ordering and inventory holding costs.
A large body of work in this area assumes either deterministic demand or stochastic demand with known distributions \cite{khouja2008review,peng2022review}, enabling global planning strategies over a fixed horizon under the assumption of replenishable upstream supply.
Within this literature, the \emph{make-to-stock} setting assumes that items are ordered in advance and held in inventory until needed, incurring holding costs until consumption, while the \emph{make-to-order} variant satisfies orders \emph{after} demand arrival, incurring delay or backlog penalties.
While these formulations are equivalent in offline settings \cite{levi2004primal}, they represent distinct operational paradigms under uncertainty.
More recently, competitive algorithms have been developed for online variants of these problems to handle unknown future demand \cite{buchbinder2013online,bienkowski2014better,moseley2025putting,azar2026online}.
In these models, demand arrives over time and the decision-maker determines when to replenish, balancing transportation and delay costs.

Through this lens, \OSSA{} can be viewed as an online, make-to-stock variant of \texttt{OWMR}, but with three key differences.
First, in \OSSA{}, demand must be satisfied immediately from on-hand inventory, and any shortfall incurs an irreversible lost-sales penalty $p$; in contrast, \texttt{OWMR} and \texttt{JRP} allow replenishment after demand realization, with costs captured through holding or backlog penalties.
Second, classical \texttt{OWMR} and \texttt{JRP} models assume access to an effectively infinite upstream supply, whereas \OSSA{} operates under a finite and unknown global supply $s$, necessitating careful rationing across sites.
Third, \OSSA{} is formulated in a fully online, adversarial setting without distributional assumptions on demand.
While such a framing may bring \OSSA{} closer in spirit to online replenishment models, the ordering of events --- where replenishment decisions only affect future demand and \emph{cannot} recover past losses --- introduces a fundamentally different set of challenges.

\paragraph{Online algorithms and resource allocation.}

\OSSA{} shares mathematical foundations with the theory of online algorithms and resource allocation \cite{borodin2005online}.
In this framework, an algorithm must make irrevocable decisions under uncertainty about future inputs, and its performance is measured relative to an optimal offline benchmark.
Canonical problems include ski rental, caching/paging \cite{sleator1985amortized}, the $k$-server problem \cite{manasse1988competitive}, online bipartite matching \cite{karp1990optimal}, and budgeted allocation models such as Adwords \cite{mehta2007adwords,devanur2009adwords,mehta2010online}.

Despite these connections, the structure of \OSSA{} precludes the direct application of standard techniques from this literature.
First, the fixed-charge transportation costs in the \OSSA{} objective induce a non-convex, step-wise objective with no meaningful marginal values.
Second, the total supply $s$ is \emph{unknown} until it is exhausted.
These factors jointly limit the effectiveness of primal-dual based techniques.
Furthermore, \OSSA{} is fundamentally stateful: allocations build persistent inventory buffers that influence future costs through lost-sales penalties.
While stateful online problems such as caching have been extensively studied, they typically utilize reusable resources (e.g., a fixed-capacity cache).
In \OSSA{}, every allocation simultaneously updates the system state and irreversibly depletes a shared resource under an unknown horizon.
This tight coupling between state evolution and terminal resource availability represents a significant gap in the current theory of online algorithms.

\paragraph{Learning-augmented algorithms.}

Many real-world supply chain and logistics systems have access to demand forecasts that are informative but imperfect.
Learning-augmented algorithms, also known as algorithms with predictions or imperfect advice, provide a principled framework for incorporating such predictions while retaining worst-case guarantees.

A central objective in this literature is to achieve three desirable properties: \emph{consistency} (performance improves when predictions are accurate), \emph{robustness} (performance remains near-optimal under adversarial predictions), and ideally \emph{smoothness} (performance degrades gracefully as prediction quality worsens).
Since the seminal work of \cite{lykouris2021competitive}, there has been substantial progress in designing learning-augmented algorithms for a wide range of online problems.
For instance, variants have been studied for ski-rental \cite{gollapudi2019online,wang2020online,angelopoulos2020online,shin2023improved} and online selection and matching problems \cite{antoniadis2020secretary,dutting2021secretaries,choo2024online,choo2025learning}.
This framework is particularly well-suited for \OSSA{}, where noisy predictions of future demand and total supply can help improve algorithmic performance.
For an overview of this growing area, we refer the reader to the survey by \cite{mitzenmacher2022algorithms}.\footnote{See also \url{https://algorithms-with-predictions.github.io/}.}

\section{Proportional Online Allocation}
\label{sec:online}

In this section, we study a threshold-proportional online allocation policy for \OSSA{}. 
The policy is parameterized by a vector $\bar{\gamma} = (\gamma_i)_{i \in [n]}$, where $\gamma_i \in [0,1]$ controls how aggressively site $i$ is replenished relative to its cumulative observed demand.

To begin, observe that we can decompose the objective (\cref{eq:main-objective}) into sitewise costs $\cost_i(\bar{\ell})$, which may be decomposed into $\cost_i(\bar{\ell}) = \mathrm{transport}_i(\bar{\ell}) + \mathrm{penalty}_i(\bar{\ell})$ as follows:
\[
\mathrm{transport}_i(\bar{\ell})
= \sum_{t \geq 1} w_i \left\lceil \frac{\ell_i^t}{c_i} \right\rceil\\
\quad
\text{and}
\quad
\mathrm{penalty}_i(\bar{\ell})
= \sum_{t \geq 1} p (d_i^t - k_i^t)_+
\]
The former scales with the number of discrete shipments while the latter represents the penalty due to unmet demand by the local inventory.
We denote cumulative demand and replenishment as $D_i^t = \sum_{t'=1}^t d_i^{t'}$ and $L_i^t = \sum_{t'=1}^t \ell_i^{t'}$, with $L_i^0 = D_i^0 = 0$.
Then, at the end of the demand arrival process, let $D_i$, $L_i$, and $s_i^{\mathrm{end}}(\bar{\ell})$ be the total demand, total supply, and leftover inventory respectively.

\textbf{Characterization of the offline optimum.\footnote{Technically, this description corresponds to a \emph{strengthening} of $\OPT$, in which the transport cost is taken to be $w_i (\ell_i^1)^\star / c_i$, rather than $w_i \lceil (\ell_i^1)^\star / c_i \rceil$. The two coincide only when $c_i = 1$. We do this to simplify the intuition: if the ceiling operator is enforced, $\OPT$ may leave a residual portion of demand at some sites unfulfilled (due to indivisibility), which obscures the clean greedy structure described here. In our analysis, we use this strengthened formulation to obtain a lower bound on $\cost(\OPT)$.}}
The offline optimum $\OPT$ observes $s$ and the entire demand sequence $\{d_i^t\}_{i \in [n], t \geq 1}$ in advance.
Since \OSSA{} has no holding costs, $\OPT$ can deliver all supply at time $t=1$ without loss of optimality, i.e., there exists an optimal solution with $(\ell_i^t)^\star = 0$ for all $t \geq 2$.
To characterize $\OPT$, define the \emph{net demand} at each site $N_i = (D_i - k_i^1)_+$, which represents the portion of demand \emph{not} covered by the initial stock.
The problem then reduces to allocating a total budget $s$ across sites to cover these net demands.
Under the ordering $\frac{w_1}{c_1} \leq \frac{w_2}{c_2} \leq \cdots \leq \frac{w_n}{c_n}$, $\OPT$ allocates supply greedily in increasing order of index.
Define the \emph{pivotal index} $i^\star = \min \{ i \in [n] : \sum_{j=1}^i N_j \geq s \}$ and the corresponding \emph{pivotal value} $\zeta = (s - \sum_{j=1}^{i^\star-1} N_j) / N_{i^\star} \in (0,1]$.
If $s \geq \sum_{j=1}^n N_j$, we define $i^\star = n$ and $\zeta = 1$.
Let us now write the optimal allocation and allocation fractions respectively as
\[
L_i^\star =
(\ell_i^1)^\star =
\begin{cases}
N_i & i < i^\star\\
\zeta \cdot N_{i^\star} & i = i^\star\\
0 & i > i^\star
\end{cases}
\qquad
\text{and}
\qquad
\gamma_i^\star =
\begin{cases}
1 & i < i^\star\\
\zeta, & i = i^\star\\
0 & i > i^\star
\end{cases}
\]
Then, we see that the vector $\bar{\gamma}^\star = (\gamma_1^\star, \ldots, \gamma_n^\star)$ is monotone non-increasing, consisting of a prefix of ones, a suffix of zeros, and at most one intermediate value $\zeta \in (0,1]$.
As illustrated in \cref{fig:gamma-assignment}, deviations from this structure translate into a tradeoff: under-allocation increases penalties due to unmet demand, while over-allocation increases transportation costs.

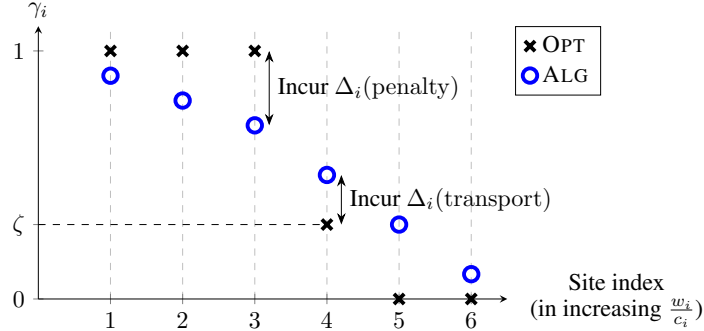
\begin{figure}[htb]
    \centering
    \resizebox{0.7\linewidth}{!}{
    \begin{tikzpicture}
\begin{axis}[
    clip=false,
    y post scale = 0.7,
    xmin=0, xmax=6.5, ymin=0, ymax=1.1,
    xtick={1, 2, 3, 4, 5, 6},
    xticklabels={$1$, $2$, $3$, $4$, $5$, $6$},
    ytick={0, 0.3, 1}, yticklabels={$0$, $\zeta$, $1$},
    xlabel={\begin{tabular}{c}Site index\\ (in increasing $\frac{w_i}{c_i}$)\end{tabular}},
    ylabel={$\gamma_i$},
    every axis x label/.style={
        at={(ticklabel* cs:1)},
        anchor=west,
    },
    every axis y label/.style={
        at={(ticklabel* cs:1)},
        anchor=south,
    },
    axis lines=left,
    xmajorgrids=true, grid style=dashed,
    scatter/classes={
        opt={mark=x, black, mark size=3pt},
        alg={mark=o, blue, mark size=3pt}
    },
    legend style={
        nodes={anchor=west},
        at={(1.2,1)},
    }
]
\addplot[scatter, only marks, ultra thick, scatter src=explicit symbolic]
coordinates {
(1, 1) [opt]
(2, 1) [opt]
(3, 1) [opt]
(4, 0.3) [opt]
(5, 0) [opt]
(6, 0) [opt]
(1, 0.9) [alg]
(2, 0.8) [alg]
(3, 0.7) [alg]
(4, 0.5) [alg]
(5, 0.3) [alg]
(6, 0.1) [alg]
};
\draw[dashed] (axis cs:0, 0.3) -- (axis cs:4, 0.3);
\draw[{Stealth}-{Stealth}] (axis cs:3.2, 1) -- node[right] {Incur $\Delta_i(\mathrm{penalty})$} (axis cs:3.2, 0.7);
\draw[{Stealth}-{Stealth}] (axis cs:4.2, 0.5) -- node[right] {Incur $\Delta_i(\mathrm{transport})$} (axis cs:4.2, 0.3);
\legend{$\OPT$, $\ALG$}
\end{axis}
\end{tikzpicture}
    }
    \caption{
    Consider an example with $n = 6$ sites.
    Suppose $\OPT$'s supply allocation proportion $\bar{\gamma}^\star$ (black crosses) has pivotal site $i^\star = 4$ with $\gamma_4^\star = \zeta$, and an online algorithm $\ALG$ allocates proportional allocation $\bar{\gamma}$ (blue circles).
    For each site $i$, $\ALG$ incurs additional sitewise penalty for unmet demand when $\gamma_i^\star > \gamma_i$ and additional sitewise transport cost when $\gamma_i^\star < \gamma_i$.
    }
    \label{fig:gamma-assignment}
\end{figure}

\textbf{A threshold-proportional online allocation policy.}
Since the online algorithm lacks knowledge of $s$ and $\{d_i^t\}_{i \in [n], t \geq 1}$, it cannot identify the pivotal index $i^\star$ nor value $\zeta$ in advance.
Instead, we seek to emulate $\OPT$ via a threshold-proportional online allocation policy called $\bar{\gamma}$-Proportional Allocation (\GPA{}, \cref{alg:proportional-allocation}).
The algorithm maintains local sitewise safety stock by requesting replenishment whenever inventory falls below the demand bound $b_i$, but only as long as cumulative replenishment does not exceed a target fraction $\gamma_i$ of observed demand.
Crucially, \GPA{} always requests shipments in full increments of capacity $c_i$, except possibly the final truncated shipment when the hub runs out of supply.
This ensures that the algorithm ``extracts'' maximum utility from each fixed cost $w_i$ incurred.
While this may result in a small amount of leftover supply at the end of the horizon, we show that this wastage is uniformly bounded at each site and can be absorbed into the additive $\beta$ term.

\begin{algorithm}[htb]
\caption{$\bar{\gamma}$-Proportional Allocation (\GPA{})}
\label{alg:proportional-allocation}
\begin{algorithmic}[1]
\Require \texttt{OSSA} problem parameters, threshold vector $\bar{\gamma} = \{ \gamma_i \}_{i \in [n]}$
\Ensure Online allocation $\bar{\ell}$
    \For{time step $t = 1, 2, \ldots$}
        \For{sites $i = 1, \ldots, n$}
            \State Demands $d_i^{t} \in \N$ arrive at each site $i$ \Comment{Incur penalty cost $p (d_i^{t} - k_i^{t})_+$}
            \State Define post-demand remainder $r_i^{t} = (k_i^{t} - d_i^{t})_+$
        \EndFor
        \State Define $R^t = \{ i \in [n]: \text{$r_i^{t} < b_i$ and $L_i^{t-1} \leq \gamma_i D_i^{t}$} \}$ \Comment{Eligible resupply sites at time step $t$}
        \For{site $i \in R_t$}
            \State Request $q_i^{t} = \left\lceil \frac{b_i - r_i^{t}}{c_i} \right\rceil$ full shipments \Comment{Number of full shipments to achieve $k_i^{t+1} \geq b_i$}
            \State Site $i$ receives $\ell_i^{t} = \min\{ s, c_i q_i^{t} \}$ units of supply \Comment{Incur transport cost $w_i \lceil \frac{\ell_i^t}{c_i} \rceil$}
            \State (Update $s = s - \ell_i^{t}$ at the hub) \Comment{Note: \GPA{} does \emph{not} see $s$}
        \EndFor
    \EndFor
    \State \Return $\{ \ell_i^{t} \}_{i \in [n], t \geq 1}$ \Comment{$\ell_i^{t} = 0$ if not initialized}
\end{algorithmic}
\end{algorithm}

The parameter $\gamma_i$ governs the aggressiveness of replenishment at site $i$.
Smaller values of $\gamma_i$ make the policy more conservative, reducing transportation costs at the expense of potentially higher penalties due to unmet demand.
Larger values of $\gamma_i$ have the opposite effect.
Thus, the analysis can be viewed as identifying threshold vectors $\bar{\gamma}$ that optimally balance these competing costs.
If $\bar{\gamma} = \bar{\gamma}^\star$, then we will exactly match $\OPT$ except possibly on the pivotal site $i^\star$.
However, due to the online nature of \OSSA{}, the identity of $i^\star$ is not known in advance.
We therefore analyze \GPA{} with respect to an arbitrary input vector $\bar{\gamma}$, and later show how to choose $\bar{\gamma}$ to achieve strong performance guarantees.
The performance of \GPA{} is characterized under two distinct regimes --- when the hub is not exhausted (\cref{lem:generic-nonexhausted}) and when it is exhausted (\cref{lem:generic-exhausted}) --- each requiring fundamentally different analyses.

\begin{restatable}[\GPA{} guarantee when hub is not exhausted]{lemma}{genericnonexhausted}
\label{lem:generic-nonexhausted}
Fix $\tau \in (0,1]$.
If $s^{\mathrm{end}} > 0$ and $\gamma_i \geq \min \left\{ 1, \frac{\tau p c_i}{w_i} \right\}$ for all $i \in [n]$, then for the online allocation $\bar{\ell}$ produced by \GPA{}, we have
\[
\cost(\bar{\ell})
\leq \left( 1 + \frac{(1-\tau)^2}{4\tau} \right) \cdot \cost(\OPT) + \sum_{i \in [n]} 3p (b_i + c_i)
\]
\end{restatable}

\begin{restatable}[\GPA{} guarantee when hub is exhausted]{lemma}{genericexhausted}
\label{lem:generic-exhausted}
Fix $\tau \in (0,1]$.
If $s^{\mathrm{end}} = 0$ and $\gamma_i \leq \min \left\{ 1, \frac{\tau p c_i}{w_i} \right\}$ for all $i \in [n]$, then for the online allocation $\bar{\ell}$ produced by \GPA{}, we have
\[
\cost(\bar{\ell})
\leq (1+\tau) \cdot \cost(\OPT) + \sum_{i \in [n]} 3 p (b_i + c_i)
\]
\end{restatable}

In the non-exhausted regime ($s^{\mathrm{end}} > 0$), every replenishment request is eventually satisfied.
This allows \GPA{} to maintain its target proportionality $L_i^t \approx \gamma_i D_i^t$ at each site, yielding direct sitewise bounds on both penalty and transportation costs.
Consider again the illustration in \cref{fig:gamma-assignment}.
For a fixed site $i$, if $\gamma_i < \gamma_i^\star$, our transportation cost remains no worse than that of $\OPT$ due to our use of full shipments, though we incur a higher penalty.
Conversely, if $\gamma_i > \gamma_i^\star$, our penalty cost is no worse than $\OPT$'s because we ensure the local stock is replenished to at least $b_i$ before the next demand arrives, albeit at the expense of higher transportation costs.
This sitewise control ensures that the cost at each site remains bounded by the chosen $\gamma_i$ parameters.

When the hub is exhausted ($s^{\mathrm{end}} = 0$), the global stock constraint may truncate eligible replenishment requests.
In this regime, the sitewise bounds used above may not hold.
Instead, we evaluate the approximation ratio by analyzing the aggregate gaps in penalty and transportation costs between \GPA{} and $\OPT$.
Because $s^{\mathrm{end}} = 0$, the total unmet demand is roughly conserved between \GPA{} and $\OPT$, so the penalty gap is roughly bounded by potential inventory wastage.
The transportation gap is more delicate: by setting $\gamma_i$ appropriately, we prove that the additional transportation cost $\Delta(\mathrm{transport})$ incurred by \GPA{} is bounded by $\tau$ times the penalty incurred by $\OPT$, plus additive terms.

The technical proofs for \cref{lem:generic-nonexhausted} and \cref{lem:generic-exhausted} rely on some fundamental structural properties of \GPA{} discussed above.
These properties are formally summarized in the following lemma.

\begin{restatable}[Structural properties of \GPA{}]{lemma}{structural}
\label{lem:structural}
For any hyperparameter $\bar{\gamma} = \{ \gamma_i \}_{i \in [n]}$ such that $\gamma_i \in [0,1]$ and let $\bar{\ell}$ be the online allocation of \GPA{}.
Then, the following properties hold.\\
\phantom{\hspace{10pt}}1. \textbf{Sitewise bounds when hub is not exhausted:}
    If $s^{\mathrm{end}} > 0$, then for each $i \in [n]$, we have
    \begin{align}
    \mathrm{transport}_i(\bar{\ell})
    &\leq \gamma_i D_i \frac{w_i}{c_i} + p (b_i + 2c_i) \label{eq:sitewise-transport-when-hub-not-exhausted}\\
    \mathrm{penalty}_i(\bar{\ell})
    &\leq (1-\gamma_i) p D_i + p c_i \label{eq:sitewise-penalty-when-hub-not-exhausted}
    \end{align}
\\
\phantom{\hspace{10pt}}2. \textbf{Sitewise transport gap:}
    If $\gamma_i \leq \frac{\tau p c_i}{w_i}$, then $\Delta(\mathrm{transport}_i)
    \leq \tau \cdot \mathrm{penalty}(\OPT) + 2p (b_i + c_i)$.
\\
\phantom{\hspace{10pt}}3. \textbf{Penalty gap when the hub is exhausted:}
    If $s^{\mathrm{end}} = 0$, then $\Delta(\mathrm{penalty})
    \leq \sum_{i \in [n]} p (b_i + c_i)$.
\end{restatable}

\section{Advice-free OSSA}
\label{sec:advice-free}

By optimizing the hyperparameter $\tau$ to balance the guarantees of the non-exhausted and exhausted regimes, we obtain our main result: set $\tau = 1/3$, then apply \cref{lem:generic-nonexhausted} and \cref{lem:generic-exhausted}.

\begin{restatable}{theorem}{onlineupperbound}
\label{thm:online-upper-bound}
\texttt{GPA} with $\gamma_i = \min\{1, \frac{p c_i}{3 w_i}\}$ achieves
$
\cost(\ALG) \leq \frac{4}{3} \cdot \cost(\OPT) + \sum_{i \in [n]} 3p (b_i + c_i)
$.
\end{restatable}

The trade-off is illustrated in \cref{fig:alpha-versus-tau}, where the crossing point of the two regime-specific curves identifies the optimal competitive ratio of $\alpha = 4/3$ when $\tau = 1/3$. 
This geometric intuition will be helpful when extending \texttt{GPA} to handle predictions in \cref{sec:learning-augmented} for varying values of $\alpha \geq 4/3$.

\begin{figure}[htb]
    \centering
    \resizebox{\linewidth}{!}{
    \begin{tikzpicture}
\begin{axis}[
    y post scale = 0.7,
    xmin=0.1, xmax=1.1, ymax=3.2,
    xlabel={$\tau$},
    ylabel={$\alpha$},
    every axis x label/.style={
        at={(ticklabel* cs:1.05)},
        anchor=west,
    },
    every axis y label/.style={
        at={(ticklabel* cs:1.05)},
        anchor=south,
    },
    axis lines=left,
    legend style={
        nodes={anchor=west},
        at={(1,1.1)},
    }
]
\addplot[name path=A, domain=0:1, dashed, ultra thick, blue!70!white, samples=100] {1 + (1-x)^2/(4*x)};
\addplot[name path=B, domain=0:1, dashed, ultra thick, red!70!white] {1+x};
\addplot[name path=C, domain=0:1, thick, samples=100, green!70!black] {max(1 + (1-x)^2/(4*x), 1+x)};
\addplot[name path=D, domain=0:1, dashed, ultra thick, orange!70!white] {1.8};

\path [name intersections={of=A and B, by=inter1}];
\node [fill=black, circle, inner sep=1.25pt] at (inter1) {};
\node [anchor=north] at (inter1) {$(\frac{1}{3}, \frac{4}{3})$};

\path [name intersections={of=A and D, by=inter2}];
\node [fill=black, circle, inner sep=1.25pt] at (inter2) {};
\node [anchor=north west, xshift=5pt] at (inter2) {$(0.2, 1.8)$};

\path [name intersections={of=B and D, by=inter3}];
\node [fill=black, circle, inner sep=1.25pt] at (inter3) {};
\node [anchor=north west] at (inter3) {$(0.8, 1.8)$};

\legend{
    {$\alpha = 1 + \frac{(1 - \tau)^2}{4 \tau}$},
    {$\alpha = 1 + \tau$},
    {$\alpha = \max \{ 1 + \frac{(1 - \tau)^2}{4 \tau}, 1 + \tau \}$}
}
\end{axis}
\end{tikzpicture}

\begin{tikzpicture}
\begin{axis}[
    y post scale = 0.7,
    xmax=1.1, ymax=1.1,
    ytick={0,1}, yticklabels={$0$, $1$},
    xlabel={$\frac{w_i}{p c_i}$},
    ylabel={$\gamma_i$},
    every axis x label/.style={
        at={(ticklabel* cs:1.05)},
        anchor=west,
    },
    every axis y label/.style={
        at={(ticklabel* cs:1.05)},
        anchor=south,
    },
    xmajorgrids=true, grid style=dashed,
    axis lines=left,
    title={$\alpha = 1.8$}
]
\pgfmathsetmacro{\myalpha}{1.8}
\addplot[name path=upper, domain=0:1, ultra thick, blue!70!white] {min(1, ( 2 * \myalpha - 2 * sqrt( \myalpha * \myalpha - \myalpha ) - 1 ) / x)};
\addplot[name path=lower, domain=0:1, ultra thick, red!70!white] {min(1, (\myalpha - 1) / x};
\addplot[gray!30] fill between [of=upper and lower];
\end{axis}
\end{tikzpicture}
    }
    \caption{
    \textbf{(Left)} Competitive ratio $\alpha$ as a function of the hyperparameter $\tau$.
    The dotted blue and red curves denote the upper bounds from \cref{lem:generic-nonexhausted} and \cref{lem:generic-exhausted} respectively.
    The solid green curve represents the overall competitive ratio, given by the pointwise maximum of the two bounds, and is minimized at $\tau = 1/3$.
    For example, if one is willing to tolerate a competitive ratio of $\alpha = 1.8$, any choice of $\tau \in [0.2, 0.8]$ suffices, as highlighted by the orange dashed line.
    \textbf{(Right)} Corresponding feasible values of $\gamma_i$ as a function of $\frac{w_i}{p c_i}$ for $\alpha = 1.8$.
    The blue curve $\gamma_i = \min\{1, (2\alpha - 1 - 2\sqrt{\alpha^2 - \alpha})\frac{p c_i}{w_i}\}$ and the red curve $\gamma_i = \min\{1, (\alpha - 1)\frac{p c_i}{w_i}\}$ are obtained by inverting the bounds from \cref{lem:generic-nonexhausted} and \cref{lem:generic-exhausted} respectively.
    Any choice of $\gamma_i$ within the shaded region yields a valid $\alpha$-approximation.
    When $\alpha = 4/3$, these curves collapse to $\gamma_i = \min\{1, \frac{p c_i}{3 w_i}\}$.
    }
    \label{fig:alpha-versus-tau}
\end{figure}
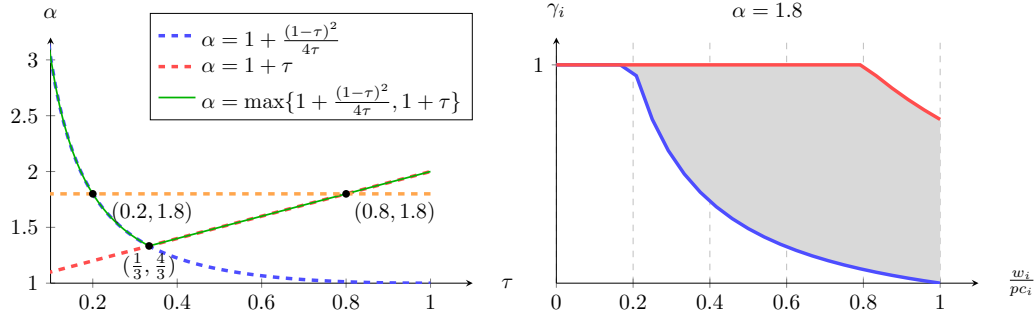

\textbf{Lower Bounds.}
We complement \cref{thm:online-upper-bound} with lower bounds establishing that the $4/3$ ratio and the structure of our additive error are essentially optimal.
More specifically, we show that no algorithm can obtain constant $\alpha$-approximation when $\beta \in o(\sum_{i=1}^n p(b_i + c_i))$, or beat the $\alpha = 4/3$ ratio without $\beta \in \Omega(s)$.\footnote{For large \OSSA{} instances with large arrival demands, the term $\sum_{i=1}^n p(b_i + c_i)$ involving sitewise constants vanishes as $\cost(\OPT)$ grows (i.e., the approximation becomes $4/3 + o(1)$), but the global supply $s$ could scale together with $\cost(\OPT)$.}
Both results are shown by constructing a distribution hard instances and then applying Yao's minimax principle \cite{yao1977probabilistic}, and they hold even when $s$ is known upfront.

\begin{restatable}{proposition}{lowerone}
For any $\eps > 0$, any randomized online algorithm $\ALG$ achieving $\E[\cost(\ALG)] \leq \alpha \cdot \cost(\OPT) + \frac{(1 - \eps)p}{8} \cdot \sum_{i \in [n]} (b_i + c_i)$ must have a non-constant competitive ratio $\alpha$.
This holds even if total supply $s$ was known upfront.
\end{restatable}

\begin{restatable}{proposition}{lowertwo}
For any $\eps > 0$, any randomized online algorithm $\ALG$ achieving $\E[\cost(\ALG)] \leq \left( \frac{4}{3} - \eps \right) \cdot \cost(\OPT) + \beta$ must have $\beta \in \Omega(s p \eps)$.
This holds even if total supply $s$ was known upfront.
\end{restatable}

\section{Learning-Augmented OSSA}
\label{sec:learning-augmented}

The characterization of $\OPT$ in \cref{sec:online} via the pivotal index $i^\star$ and the pivotal value $\zeta$ guides our choice of prediction model as we strive to design algorithms that can achieve 1-consistency under perfect predictions, i.e., $\alpha = 1$.
The following proposition rules out certain forms of advice as insufficient to effectively approximate $\OPT$, even when they are perfectly accurate.

\begin{restatable}{proposition}{advicetooweak}
\label{prop:advice-too-weak}
No (possibly randomized) online algorithm $\ALG$ can guarantee $\E[\cost(\ALG)] \leq \cost(\OPT) + O(\sum_{i=1}^n p (b_i + c_i))$ when given access only to \emph{perfect} predictions of any of the following:\\
\phantom{\hspace{10pt}}$\bullet$ Sitewise total demand $\{ D_i \}_{i \in [n]}$\\
\phantom{\hspace{10pt}}$\bullet$ One-step-lookahead demand at each site $\{d_i^{t+1}\}_{i \in [n]}$, for all time steps $t \geq 1$\\
\phantom{\hspace{10pt}}$\bullet$ Total supply $s$ and total demand $\sum_{i=1}^n D_i$
\end{restatable}

Motivated by \cref{prop:advice-too-weak}, we assume access to predictions $\hat{s}$ and $\{\hat{D}_i\}_{i \in [n]}$ of the total supply and total demand at each site.
We use these predictions in \cref{alg:prediction} to construct a threshold vector $\bar{\gamma}$, which is then used as input to \cref{alg:proportional-allocation}.
The key idea is to mimic the structure of $\OPT$ using the predicted quantities.
To hedge against potentially inaccurate predictions, we introduce a \emph{distrust hyperparameter} $\lambda \in (0, \frac{1}{3}]$.
Smaller values of $\lambda$ correspond to placing greater trust in the predictions, while $\lambda = \frac{1}{3}$ recovers the fully robust, advice-free setting.

\begin{algorithm}[htb]
\caption{Prediction-guided $\bar{\gamma}$ construction}
\label{alg:prediction}
\begin{algorithmic}[1]
\Require \texttt{OSSA} problem parameters, predictions $\hat{s}$ and $\{\hat{D}_i\}_{i \in [n]}$, distrust hyperparameter $\lambda \in \left(0,\frac{1}{3} \right]$
\Ensure Threshold vector $\bar{\gamma} = \{ \gamma_i \}_{i \in [n]}$
    \State For each $i \in [n]$, define net demand beyond initial stock  $\hat{N}_i = (\hat{D}_i - k_i^1)_+$
    \If{$\hat{s} \geq \sum_{j=1}^n \hat{N}_j$}
        Define $\hat{i}^\star = n$ and $\hat{\zeta} = 1$
        \Comment{Estimate $i^\star$ and $\zeta$}
    \Else{}
        Define $\hat{i}^\star = \min \{ i \in [n] : \sum_{j=1}^i \hat{N}_j \geq \hat{s} \}$
        and $\hat{\zeta} = (\hat{s} - \sum_{j=1}^{\hat{i}^\star-1} \hat{N}_j)_+ / \hat{N}_{\hat{i}^\star} \in (0,1]$
    \EndIf
    \State Define $\tau = ( \sqrt{1 + \lambda} -\sqrt{\lambda} )^2$ \Comment{For $\lambda \in (0, \frac{1}{3}]$, we have $\lambda \leq \tau \in [\frac{1}{3}, 1)$}
    \State Define $\texttt{lower}_i = \min \{ 1, \frac{\lambda p c_i}{w_i} \}$, $\texttt{upper}_i = \min \{ 1, \frac{\tau p c_i}{w_i} \}$, and \Comment{Clip $\gamma_{\hat{i}^\star} = \hat{\zeta}$ within bounds}
    \[
    \gamma_i
    = \begin{cases}
    \texttt{upper}_i & \text{if $i \in \{1, \ldots, \hat{i}^\star - 1\}$}\\
    \min\{\texttt{upper}_i, \max\{\hat{\zeta}, \texttt{lower}_i \}\} & \text{if $i = \hat{i}^\star$}\\
    \texttt{lower}_i & \text{if $i \in \{\hat{i}^\star + 1, \ldots, n\}$}
    \end{cases}
    \]
    \State \Return $\bar{\gamma} = \{ \gamma_i \}_{i \in [n]}$
\end{algorithmic}
\end{algorithm}

Our choice of \texttt{upper} and \texttt{lower} correspond to the two lines in the right plot in \cref{fig:alpha-versus-tau}, and robustness follows directly by ensuring that $\bar{\gamma}$ lies in the gray area.
The following theorem establishes the performance guarantees of the learning-augmented algorithm $\LAGPA(\lambda)$ attained by running \GPA{} using $\bar{\gamma}$ generated from \cref{alg:prediction}.

\begin{restatable}{theorem}{learningaugmented}
\label{thm:learning-augmented}
Consider the predictions $\{\hat{s}$, $\{ \hat{D}_i \}_{i \in [n]}\}$ for the given \texttt{OSSA} instance with prediction error $\eta = | s - \hat{s} | + \sum_{i=1}^n | D_i - \hat{D}_i | \geq 0$.
Given a distrust hyperparameter $\lambda \in (0, \frac{1}{3}]$, the learning-augmented algorithm $\LAGPA(\lambda)$ uses the above predictions and has the following guarantees:\\
\phantom{\hspace{10pt}}1. Robustness: $\cost(\LAGPA(\lambda)) \leq (1 + \frac{(1 - \lambda)^2}{4 \lambda}) \cdot \cost(\OPT) + O(\sum_{i \in [n]} (b_i + c_i))$\\
\phantom{\hspace{10pt}}2. Consistency / Smoothness: $\cost(\LAGPA(\lambda)) \leq (1 + \lambda) \cdot \cost(\OPT) + 3 \eta p + O( \sum_{i \in [n]} p(b_i + c_i))$
\end{restatable}
\vspace{-10pt}
\begin{proof}[Proof sketch]
Our analysis once again separates the regimes of $s^{\mathrm{end}} > 0$ and $s^{\mathrm{end}} = 0$.
The robustness guarantee follows directly from \cref{lem:generic-nonexhausted} and \cref{lem:generic-exhausted} under our choice of $\bar{\gamma}$; see \cref{fig:alpha-versus-tau} for intuition.
Meanwhile, the consistency guarantee is obtained by showing that the prediction-induced deviation satisfies $\sum_{i=1}^nN_i \lvert \gamma_i^\star - \gamma_i \rvert \leq 3 \eta$, where $\bar{\gamma}^\star$ is an optimal offline threshold vector.
\end{proof}

The first guarantee ensures robustness to adversarial predictions, matching the worst-case bound of the advice-free guarantees in \cref{thm:online-upper-bound}.
The second guarantee shows that the algorithm smoothly interpolates toward optimal performance as the prediction error decreases.
In particular, when $\eta = 0$ and $\lambda \to 0$, we obtain $\cost(\ALG) \leq \cost(\OPT) + O(\sum_{i \in [n]} p(b_i + c_i))$, achieving 1-consistency of $\alpha = 1$ up to the unavoidable additive term.
The inclusion of the distrust hyperparameter $\lambda$ is not merely a modeling convenience, but a fundamental necessity for learning-augmented algorithms in environments where the reliability of a prediction is unknown, e.g., \cite{gollapudi2019online,wang2020online,angelopoulos2020online,shin2023improved}; see \cref{sec:conclusion} for further discussion.
Furthermore, our additive dependence on the prediction error is standard in this literature, e.g., see \cite{wang2020online}.

Finally, we complement \cref{thm:learning-augmented} by showing that the achieved tradeoff between consistency and robustness is Pareto optimal: any improvement in one necessarily degrades the other.

\begin{restatable}{proposition}{paretopoint}
\label{prop:pareto-point}
Consider the predictions $\{\hat{s}$, $\{ \hat{D}_i \}_{i \in [n]}\}$ for the given \texttt{OSSA} instance with prediction error $\eta = | s - \hat{s} | + \sum_{i=1}^n | D_i - \hat{D}_i | \geq 0$.
For any $\lambda \in (0, \frac{1}{3}]$ and $0 < \eps < 1$, no (possibly randomized) online algorithm $\ALG$ with access to the above predictions can simultaneously achieve:\\
\phantom{\hspace{10pt}}1. For all $\eta \geq 0$, we have $\E[\cost(\ALG)] \leq (1 + \frac{(1 - \lambda)^2}{4 \lambda}) \cdot \cost(\OPT) + O( \sum_{i \in [n]} p(b_i + c_i))$\\
\phantom{\hspace{10pt}}2. If $\eta = 0$, then $\E[\cost(\ALG)] \leq (1 + \lambda \eps) \cdot \cost(\OPT) + O( \sum_{i \in [n]} p(b_i + c_i))$
\end{restatable}

\section{Experiments}
\label{sec:experiments}

We empirically validate $\texttt{GPA}$ against natural baselines while demonstrating the impact of $\lambda$ and $\eta$ in the learning-augmented setting; see \cref{fig:experiments}.
Across all experiments, \texttt{GPA} outperforms the other methods, especially when global supply is scarce.
We provide full experimental details, introduce the baselines, show additional synthetic experiments, and discuss the experimental results in \cref{sec:appendix-experiments}.

\begin{figure}[htb]
    \centering
    \includegraphics[width=\linewidth]{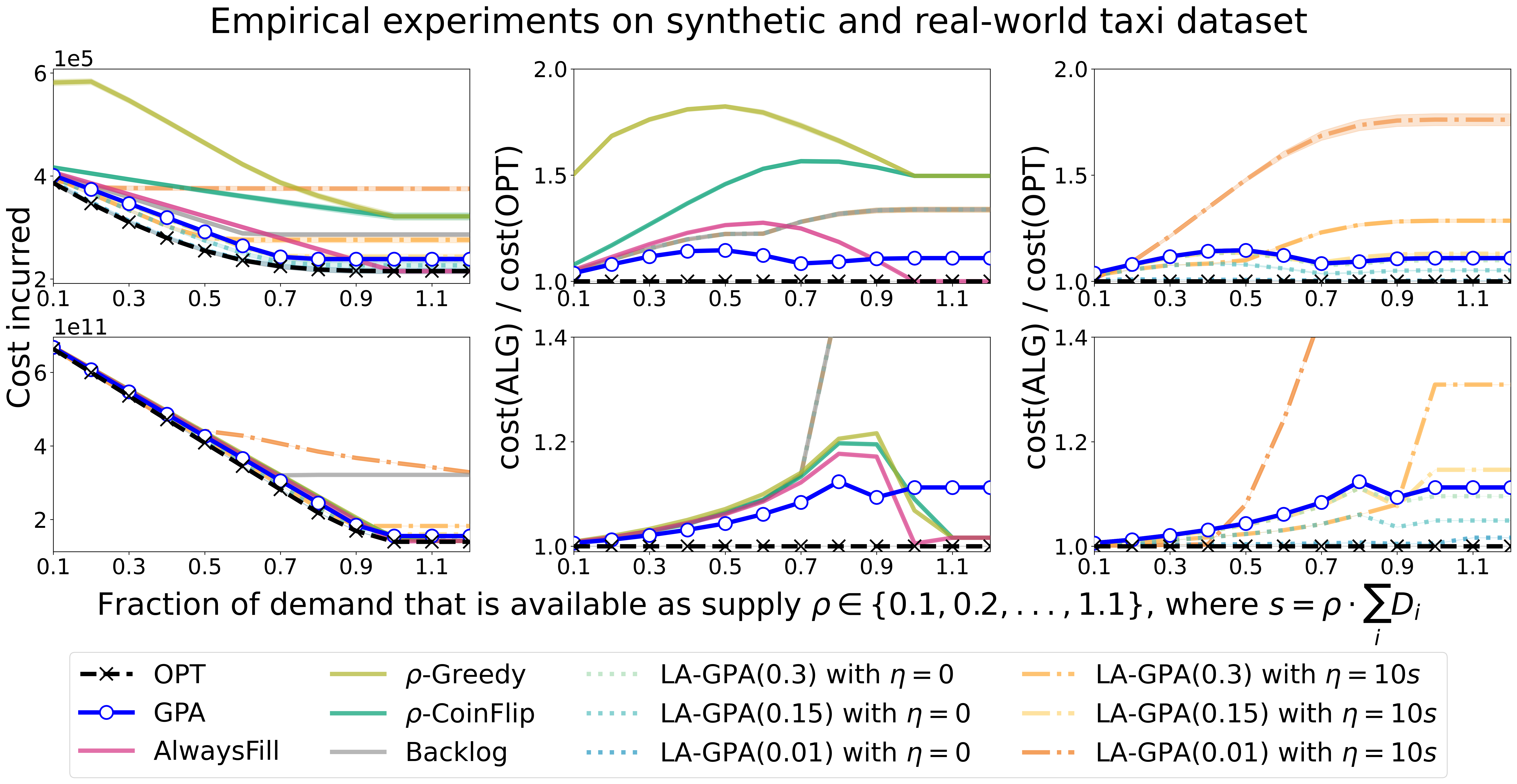}
    \caption{\textbf{Top row}: Synthetic experiments over 30 runs. \textbf{Bottom row}: Real-world taxi dataset \cite{nyc_tlc_trip_data} repurposed for \OSSA{}. The left column records cost incurred by all policies as a function of the total supply available. Meanwhile, the ratios $\cost(\ALG) / \cost(\OPT)$ are separated across two plots for visual clarity. The middle plot compares the ratio of \texttt{GPA} (blue) with other baselines while the right plot compares \texttt{GPA} with different distrust hyperparameters $\lambda \in (0, 1/3]$ and advice quality $\eta \geq 0$.}
    \label{fig:experiments}
\end{figure}

\section{Conclusion, discussions, and future directions}
\label{sec:conclusion}

We introduced the \OSSA{} problem, a model for online prepositioning from a shared and potentially unknown supply.
By bridging the gap between classical multi-echelon inventory theory and modern online allocation, \OSSA{} provides a rigorous model for critical systems such as humanitarian aid and medical supply chains.
Our main contribution is a deterministic threshold-proportional policy \GPA{}, along with a tight characterization of its performance: \GPA{} achieves a \(4/3\)-approximation to the offline optimum, up to an additive term independent of the total supply.
We complemented this result with matching lower bounds showing that both the multiplicative ratio and the additive-error dependence are essentially unavoidable.
We also extended the framework to the learning-augmented setting, showing how imperfect predictions of total supply and sitewise demand can improve performance while preserving worst-case robustness.
These results suggest that carefully tuned proportional rules can be remarkably effective in the face of deep uncertainty.

Several promising directions for future research remain.
While our model assumes a single hub and item type, real-world logistics often involve hierarchical networks and multiple commodities sharing limited transportation capacity.
Extending to such generalized settings present an interesting challenge.
Furthermore, while we treated the total supply as fixed but unknown, studying settings with stochastic replenishment or the option to purchase emergency supply at a premium would broaden the model's applicability.
Finally, our learning-augmented approach utilizes a fixed distrust hyperparameter $\lambda$ to navigate the consistency-robustness tradeoff.
In settings where advice can be partially validated online, such as stochastic or distributional demand models, it may be possible to develop adaptive ``Test-and-Act'' mechanisms to replace the need for a static $\lambda$ distrust hyperparameter, e.g., see \cite{choo2023active,choo2024online,bhattacharyya2025learning,bhattacharyya2025product,choo2024learning}.
Developing such adaptive, advice-aware algorithms for stateful supply-allocation problems is an interesting direction for future work.

\begin{ack}
This work was supported by ONR MURI N00014-24-1-2742.
\end{ack}

\bibliography{refs}

@article{rao2017immunization,
  title={{Immunization supply chains: Why they matter and how they are changing}},
  author={Rao, Raja and Schreiber, Benjamin and Lee, Bruce Y.},
  journal={Vaccine},
  year={2017}
}

@article{balcik2008facility,
  title={{Facility location in humanitarian relief}},
  author={Balcik, B. and Beamon, B. M.},
  journal={International Journal of Logistics},
  year={2008}
}

@book{sherbrooke2004optimal,
  title={{Optimal Inventory Modeling of Systems: Multi-Echelon Techniques}},
  author={Sherbrooke, Craig C.},
  year={2004},
  publisher={Springer}
}

@article{zipkin2008old,
  title={{Old and New Methods for Lost-Sales Inventory Systems}},
  author={Zipkin, Paul},
  journal={Operations Research},
  year={2008}
}

@article{tezuka2026impact,
  title={{The impact of U.S. foreign aid reduction on global health}},
  author={Tezuka, Tomomi and Ito, Naomi and Takahashi, Kenzo},
  journal={Tropical Medicine and Health},
  volume={54},
  number={1},
  pages={33},
  year={2026},
  publisher={Springer}
}

@misc{betterworldcampaign,
  author={{Better World Campaign}},
  title={{The Impact of Foreign Aid Cuts}},
  note={Available at \url{https://betterworldcampaign.org/impact-of-foreign-assistance-cuts}}
}

@article{arkin1989computational,
  title={{Computational Complexity of Uncapacitated Multi-Echelon Production Planning Problems}},
  author={Arkin, Esther and Joneja, Dev and Roundy, Robin},
  journal={Operations Research Letters},
  year={1989}
}

@article{roundy198598,
  title={{98\%-Effective Integer-Ratio Lot-Sizing for One-Warehouse Multi-Retailer Systems}},
  author={Roundy, Robin},
  journal={Management Science},
  year={1985}
}

@article{mehta2007adwords,
  title={{AdWords and generalized online matching}},
  author={Mehta, Aranyak and Saberi, Amin and Vazirani, Umesh and Vazirani, Vijay},
  journal={Journal of the ACM (JACM)},
  year={2007}
}

@inproceedings{devanur2009adwords,
  title={{The adwords problem: online keyword matching with budgeted bidders under random permutations}},
  author={Devanur, Nikhil R. and Hayes, Thomas P.},
  booktitle={Conference on Electronic Commerce (EC)},
  year={2009}
}

@article{mehta2010online,
  title={{Online Matching and Ad Allocation}},
  author={Mehta, Aranyak},
  journal={Foundations and Trends in Theoretical Computer Science},
  year={2013}
}

@article{federgruen1992joint,
  title={{The joint replenishment problem with general joint cost structures}},
  author={Federgruen, Awi and Zheng, Yu-Sheng},
  journal={Operations Research},
  year={1992}
}

@article{buchbinder2013online,
  title={{Online Make-to-Order Joint Replenishment Model: Primal-Dual Competitive Algorithms}},
  author={Buchbinder, Niv and Kimbrel, Tracy and Levi, Retsef and Makarychev, Konstantin and Sviridenko, Maxim},
  journal={Operations Research},
  year={2013}
}

@inproceedings{levi2004primal,
  title={{Primal-dual algorithms for deterministic inventory problems}},
  author={Levi, Retsef and Roundy, Robin and Shmoys, David B.},
  booktitle={Symposium on Theory of Computing},
  year={2004}
}

@article{khouja2008review,
  title={{A review of the joint replenishment problem literature: 1989–2005}},
  author={Khouja, Moutaz and Goyal, Suresh},
  journal={European Journal of Operational Research},
  year={2008}
}

@article{peng2022review,
  title={{A review of the joint replenishment problem from 2006 to 2022}},
  author={Peng, Lu and Wang, Lin and Wang, Sirui},
  journal={Management System Engineering},
  year={2022}
}

@inproceedings{moseley2025putting,
  title={{Putting Off the Catching Up: Online Joint Replenishment Problem with Holding and Backlog Costs}},
  author={Moseley, Benjamin and Niaparast, Aidin and Ravi, R.},
  booktitle={Symposium on Discrete Algorithms (SODA)},
  year={2025}
}

@inproceedings{azar2026online,
  title={{Online Joint Replenishment Problem with Arbitrary Holding and Backlog Costs}},
  author={Azar, Yossi and Lewkowicz, Shahar},
  booktitle={Symposium on Discrete Algorithms (SODA)},
  year={2026}
}

@book{borodin2005online,
  title={{Online Computation and Competitive Analysis}},
  author={Borodin, Allan and El-Yaniv, Ran},
  year={2005},
  publisher={Cambridge University Press}
}

@article{levi2008constant,
  title={{A Constant Approximation Algorithm for the One-Warehouse Multiretailer Problem}},
  author={Levi, Retsef and Roundy, Robin and Shmoys, David and Sviridenko, Maxim},
  journal={Management Science},
  year={2008}
}

@article{lykouris2021competitive,
  title={{Competitive Caching with Machine Learned Advice}},
  author={Lykouris, Thodoris and Vassilvitskii, Sergei},
  journal={Journal of the ACM (JACM)},
  volume={68},
  number={4},
  pages={1--25},
  year={2021},
  publisher={ACM New York, NY}
}

@inproceedings{gollapudi2019online,
  title={{Online Algorithms for Rent-or-Buy with Expert Advice}},
  author={Gollapudi, Sreenivas and Panigrahi, Debmalya},
  booktitle={International Conference on Machine Learning (ICML)},
  pages={2319--2327},
  year={2019},
}

@article{wang2020online,
  title={{Online Algorithms for Multi-shop Ski Rental with Machine Learned Advice}},
  author={Wang, Shufan and Li, Jian and Wang, Shiqiang},
  journal={Neural Information Processing Systems (NeurIPS)},
  volume={33},
  pages={8150--8160},
  year={2020}
}

@inproceedings{angelopoulos2020online,
  title={{Online Computation with Untrusted Advice}},
  author={Angelopoulos, Spyros and D{\"u}rr, Christoph and Jin, Shendan and Kamali, Shahin and Renault, Marc},
  booktitle={Innovations in Theoretical Computer Science Conference (ITCS)},
  year={2020},
  organization={Schloss Dagstuhl-Leibniz-Zentrum f{\"u}r Informatik}
}

@inproceedings{shin2023improved,
  title={{Improved Learning-Augmented Algorithms for the Multi-Option Ski Rental Problem via Best-Possible Competitive Analysis}},
  author={Shin, Yongho and Lee, Changyeol and Lee, Gukryeol and An, Hyung-Chan},
  booktitle={International Conference on Machine Learning (ICML)},
  year={2023}
}

@article{antoniadis2020secretary,
  title={{Secretary and Online Matching Problems with Machine Learned Advice}},
  author={Antoniadis, Antonios and Gouleakis, Themis and Kleer, Pieter and Kolev, Pavel},
  journal={Neural Information Processing Systems (NeurIPS)},
  volume={33},
  pages={7933--7944},
  year={2020}
}

@inproceedings{dutting2021secretaries,
  title={{Secretaries with Advice}},
  author={D{\"u}tting, Paul and Lattanzi, Silvio and Paes Leme, Renato and Vassilvitskii, Sergei},
  booktitle={ACM Conference on Economics and Computation},
  pages={409--429},
  year={2021}
}

@inproceedings{choo2024online,
  title={{Online bipartite matching with imperfect advice}},
  author={Choo, Davin and Gouleakis, Themis and Ling, Chun Kai and Bhattacharyya, Arnab},
  booktitle={International Conference on Machine Learning (ICML)},
  pages = {8762--8781},
  year={2024},
}

@article{choo2025learning,
  title={Learning-Augmented Online Bipartite Fractional Matching},
  author={Choo, Davin and Jin, Billy and Shin, Yongho},
  journal={Conference on Neural Information Processing Systems (NeurIPS)},
  year={2025}
}

@article{mitzenmacher2022algorithms,
  author={Mitzenmacher, Michael and Vassilvitskii, Sergei},
  title={{Algorithms with Predictions}},
  journal={Communications of the ACM},
  publisher={Association for Computing Machinery (ACM)},
  volume={65},
  number={7},
  pages={33--35},
  year={2022}
}

@article{sleator1985amortized,
  title={{Amortized efficiency of list update and paging rules}},
  author={Sleator, Daniel D. and Tarjan, Robert E.},
  journal={Communications of the ACM},
  volume={28},
  number={2},
  pages={202--208},
  year={1985},
  publisher={ACM New York, NY, USA}
}

@inproceedings{manasse1988competitive,
  title={{Competitive algorithms for on-line problems}},
  author={Manasse, Mark and McGeoch, Lyle and Sleator, Daniel},
  booktitle={Symposium on Theory of Computing},
  year={1988}
}

@inproceedings{karp1990optimal,
  title={{An Optimal Algorithm for On-line Bipartite Matching}},
  author={Karp, Richard M. and Vazirani, Umesh V. and Vazirani, Vijay V.},
  booktitle={Symposium on Theory of Computing (STOC)},
  year={1990}
}

@inproceedings{yao1977probabilistic,
  title={{Probabilistic computations: Toward a unified measure of complexity}},
  author={Yao, Andrew Chi-Chin},
  booktitle={Annual Symposium on Foundations of Computer Science (FOCS)},
  year={1977}
}

@inproceedings{bienkowski2014better,
  title={{Better Approximation Bounds for the Joint Replenishment Problem}},
  author={Bienkowski, Marcin and Byrka, Jaroslaw and Chrobak, Marek and Je{\.z}, {\L}ukasz and Nogneng, Dorian and Sgall, Ji{\v{r}}{\'\i}},
  booktitle={Symposium on Discrete Algorithms (SODA)},
  year={2014}
}

@misc{nyc_tlc_trip_data,
  title={TLC Trip Record Data},
  author={{NYC Taxi and Limousine Commission}},
  note={Available at \url{https://www.nyc.gov/site/tlc/about/tlc-trip-record-data.page}}
}

@phdthesis{choo2024learning,
  title={{Learning Probabilistic and Causal Models with (out) Imperfect Advice}},
  author={Choo, Davin},
  year={2024},
  school={National University of Singapore (NUS)}
}

@inproceedings{bhattacharyya2025product,
  title={{Product distribution learning with imperfect advice}},
  author={Bhattacharyya, Arnab and Choo, Davin and John, Philips George and Gouleakis, Themis},
  booktitle={Neural Information Processing Systems (NeurIPS)},
  year={2025}
}

@inproceedings{bhattacharyya2025learning,
  title={{Learning multivariate Gaussians with imperfect advice}},
  author={Bhattacharyya, Arnab and Choo, Davin and John, Philips George and Gouleakis, Themis},
  booktitle={International Conference on Machine Learning (ICML)},
  year={2025}
}

@inproceedings{choo2023active,
  title={{Active causal structure learning with advice}},
  author={Choo, Davin and Gouleakis, Themistoklis and Bhattacharyya, Arnab},
  booktitle={International Conference on Machine Learning (ICML)},
  year={2023}
}
\bibliographystyle{alpha}


\appendix
\section{Deferred proofs}
\label{sec:appendix-proofs}

Here, we give the formal proofs that we have deferred from the main paper.

Let us introduce additional notation that will be useful.
For each site $i \in [n]$, $D_i^m(\bar{\ell})$, $D_i^u(\bar{\ell})$, and $s_i^{\mathrm{end}}$ be the total met demand and total unmet demand, and amount of remaining supply at the end of the demand arrival process respectively.

\subsection{Proofs for \texorpdfstring{\cref{sec:advice-free}}{Section 5}}

We will prove \cref{lem:structural} before \cref{lem:generic-nonexhausted} and \cref{lem:generic-exhausted}, relying on the following helper lemmas.

\begin{lemma}[Bounded terminal inventory]
\label{lem:bounded-terminal-inventory}
Fix any hyperparameter $\bar{\gamma} = \{ \gamma_i \}_{i \in [n]}$ such that $\gamma_i \in [0,1]$ and let $\bar{\ell}$ be the online allocation of $\bar{\gamma}\mathrm{PA}$.
For every site $i \in [n]$, we have $s_i^{\mathrm{end}}(\bar\ell) \leq b_i + c_i$.
\end{lemma}
\begin{proof}
Fix an arbitrary resupply step $t \geq 1$ with current leftover stock $r_i^t < b_i$.
$\bar{\gamma}\mathrm{PA}$ requests and receives at most $c_i q_i^{t}$ units of supply, where $q_i^{t} = \left\lceil \frac{b_i - r_i^{t}}{c_i} \right\rceil$ is the number of full shipments requested.
Thus,
\[
k_i^{t+1}
\leq r_i^t + c_i q_i^{t}
= r_i^t + c_i \left\lceil \frac{b_i - r_i^{t}}{c_i} \right\rceil
\leq r_i^t + c_i \left( \frac{b_i - r_i^{t}}{c_i} + 1 \right)
= b_i + c_i
\]
Since this holds for any arbitrary time step, we have $s_i^{\mathrm{end}}(\bar\ell) \leq b_i + c_i$.
\end{proof}

\begin{lemma}[Upper bound on cumulative allocation]
\label{lem:upper-bound-on-cumulative}
Fix any hyperparameter $\bar{\gamma} = \{ \gamma_i \}_{i \in [n]}$ such that $\gamma_i \in [0,1]$ and let $\bar{\ell}$ be the online allocation of $\bar{\gamma}\mathrm{PA}$.
For every site $i \in [n]$ and time step $t \in \N$, we have $L_i^t \leq \gamma_i D_i^t + b_i+c_i$.
In particular, we have $L_i \leq \gamma_i D_i + b_i + c_i$.
\end{lemma}
\begin{proof}
We prove $L_i^t \leq \gamma_i D_i^t + b_i + c_i$ by induction over $t \in \N$.
The claim follows by applying the induction repeatedly across all time steps, where $L_i^t$ becomes $L_i$ and $D_i^t$ becomes $D_i$.

At the base case, when $t = 0$, we have $L_i^0 = D_i^0 = 0$ across all sites $i \in [n]$.

Now, consider an arbitrary time step $t \geq 1$.
Recall that $\bar{\gamma}\mathrm{PA}$ only resupplies sites in the set $R^t = \{ i \in [n]: \text{$r_i^{t} < b_i$ and $L_i^{t-1} \leq \gamma_i D_i^{t}$} \}$, requesting $c_i q_i^t$ units of resupply for each site $i \in R^t$.
We show that the induction step holds for both $i \not\in R^t$ and $i \in R^t$.

If $i \not\in R^t$, then
\begin{align*}
L_i^{t}
&= L_i^{t-1} \tag{Since $i \not\in R^t$}\\
&\leq \gamma_i D_i^{t-1} + c_i \tag{Induction hypothesis}\\
&\leq \gamma_i D_i^{t} + c_i \tag{Since $D_i^{t} \geq D_i^{t-1}$}\\
&\leq \gamma_i D_i^{t} + b_i + c_i \tag{Since $b_i \geq 0$}
\end{align*}

If $i \in R^t$, then
\begin{align*}
L_i^{t}
&\leq L_i^{t-1} + c_i q_i^{t-1} \tag{Since $\bar{\gamma}\mathrm{PA}$ requests $\ell_i^{t-1} \leq c_i q_i^{t-1}$ units of resupply at $t-1$}\\
&\leq \gamma_i D_i^t + c_i q_i^{t-1} \tag{Since $i \in R^t$}\\
&\leq \gamma_i D_i^t + c_i \left( \frac{b_i - r_i^{t-1}}{c_i} + 1 \right) \tag{Since $q_i^{t-1} = \left\lceil \frac{b_i - r_i^{t-1}}{c_i} \right\rceil$}\\
&= \gamma_i D_i^t + b_i - r_i^{t-1} + c_i\\
&\leq \gamma_i D_i^t + b_i + c_i \tag{Since $r_i^{t-1} \geq 0$}
\end{align*}
\end{proof}
    
\begin{lemma}[Lower bound on cumulative allocation]
\label{lem:lower-bound-on-cumulative-when-hub-not-exhausted}
Fix any hyperparameter $\bar{\gamma} = \{ \gamma_i \}_{i \in [n]}$ such that $\gamma_i \in [0,1]$ and let $\bar{\ell}$ be the online allocation of $\bar{\gamma}\mathrm{PA}$.
If $s^{\mathrm{end}} > 0$, then we have $L_i \geq \gamma_i D_i$ for all sites $i \in [n]$.
\end{lemma}
\begin{proof}
Fix a site $i \in [n]$.
Recall that $\bar{\gamma}\mathrm{PA}$ only resupplies sites in the set $R^t = \{ i \in [n]: \text{$r_i^{t} < b_i$ and $L_i^{t-1} \leq \gamma_i D_i^{t}$} \}$, requesting $c_i q_i^t$ units of resupply for each site $i \in R^t$.
Under $s^{\mathrm{end}} > 0$, site $i$ will \emph{always} be receive the full $c_i q_i^t$ units of resupply when $i \in R^t$.

We now prove inductively that $L_i^t \geq \gamma_i D_i^t  + (k_i^{t+1}-b_i)_{+}$ by induction over $t \in \N$. Hence, when the algorithm ends, $L_i^t =L_i$ and $D_i^t=D_i$ and since $(k_i^{t+1} - b_i)_{+} \geq 0$, we obtain $L_i \geq \gamma_i D_i $. 

At the base case, when $t = 0$, we have $L_i^0 = D_i^0$ and $k_i^1 = b_i \geq 0$, so
\[
L_i^0
\geq \gamma_i D_i^0 + (k_i^1 - b_i)_{+}
= \gamma_i D_i^0
\]

Now, consider an arbitrary time step $t \geq 1$.
We analyze the following three cases separately:
\begin{enumerate}
    \item $r_i^t \geq b_i$
    \item $r_i^t < b_i$ and $L^{t-1}_i \geq \gamma_i D^t_i$
    \item $r_i^t < b_i$ and $L^{t-1}_i < \gamma_i D^t_i$ 
\end{enumerate}

\paragraph{Case 1: $r_i^t \geq b_i$.}
In this case, $i \not\in R^t$, so $L_i^t = L_i^{t-1}$ and $k_i^{t+1} = r_i^t$.
Moreover, since $r_i^t \geq b_i$, we have $r_i^t = k_i^t - d_i^t$, and so $k_i^{t+1} = r_i^t = k_i^t - d_i^t$.
Thus,
\begin{align*}
L_i^t
&= L_i^{t-1}\\
&\geq \gamma_i D_i^{t-1}  + (k_i^t-b_i)_{+} \tag{By induction hypothesis}\\
&= \gamma_i D_i^{t-1}  + (k_i^t-b_i) \tag{Since \(b_i\leq r_i^t \leq k_i^t\)}\\
&= \gamma_i (D_i^{t}-d_i^t) + (k_i^{t+1}+d_i^t-b_i) \tag{Since $D_i^{t-1} = D_i^t - d_i^t$ and $k_i^t = k_i^{t+1} + d_i^t$}\\
&= \gamma_i D_i^t + (k_i^{t+1}-b_i) + (1-\gamma_i)d_i^t\\
&\geq \gamma_i D_i^t + (k_i^{t+1}-b_i) \tag{(Since $\gamma_i \leq 1$ and $d_i^t \geq 0$)}\\
&= \gamma_i D_i^t + (k_i^{t+1}-b_i)_{+} \tag{Since $k_i^{t+1}=r_i^t\geq b_i$}
\end{align*}

\paragraph{Case 2: $r_i^t < b_i$ and $L^{t-1}_i \geq \gamma_i D^t_i$.}
In this case, $i \not\in R^t$, so $L_i^t = L_i^{t-1}$ and $k_i^{t+1} = r_i^t < b_i$.
Therefore,
\begin{align*}
L_i^t
&= L_i^{t-1}\\
&\geq \gamma_i D_i^t \tag{By the case premise}\\
&= \gamma_i D_i^t + (k_i^{t+1}-b_i)_{+} \tag{Since $k_i^{t+1} = r_i^t < b_i$}
\end{align*}

\paragraph{Case 3: \(r_i^t < b_i\) and \(L^{t-1}_i < \gamma_i D^t_i\).}
In this case, $i \in R^t$.
Let the resupply amount be $\ell_i^t$.
Since the resupply brings the stock level to at least $b_i$, we have
\begin{equation}\label{eq:str1}
k_i^{t+1} \geq b_i
\end{equation}

We now split into two subcases depending on whether \(k_i^t \geq b_i\) or
\(k_i^t < b_i\).

\begin{enumerate}
    \item Suppose $k_i^t \geq b_i$.
    Since $r_i^t=(k_i^t-d_i^t)_+$, and in this case $r_i^t > 0$, we have
    $r_i^t=k_i^t-d_i^t$.

    Therefore,
    \begin{align*}
    L_i^t
    &= L_i^{t-1} + \ell_i^t\\
    &\geq \gamma_i D_i^{t-1} + k_i^t - b_i + \ell_i^t \tag{By induction hypothesis}\\
    &= \gamma_i(D_i^t - d_i^t) + d_i^t + r_i^t - b_i + \ell_i^t \tag{Since $D_i^t = D_i^{t-1} + d_i^t$ and $r_i^t=k_i^t-d_i^t$}\\
    &= \gamma_i D_i^t +(1-\gamma_i) d_i^t + k_i^{t+1} - b_i \tag{Since $k_i^{t+1} = r_i^t + \ell_i^t$}\\
    &\geq \gamma_i D_i^t + k_i^{t+1}-b_i \tag{Since $\gamma_i \leq 1$ and $d_i^t \geq 0$}\\
    &= \gamma_i D_i^t+(k_i^{t+1}-b_i)_+ \tag{By \cref{eq:str1}}
    \end{align*}

    \item Suppose $k_i^t < b_i$.
    If $r_i^t = 0$, then $d_i^t \leq b_i = b_i - r_i^t$.
    Meanwhile, if $r_i^t > 0$, then $r_i^t = k_i^t - d_i^t < b_i - d_i^t$.
    In either case, regardless of whether $r_i^t = 0$ or $r_i^t > 0$, we have
    \begin{equation}\label{eq:str4}
        b_i-r_i^t \geq d_i^t .
    \end{equation}
    Therefore,
    \begin{align*}
    L_i^t
    &= L_i^{t-1} + \ell_i^t\\
    &\geq \gamma_i D_i^{t-1} + \ell_i^t \tag{By the induction hypothesis}\\
    &= \gamma_i(D_i^t - d_i^t) + \ell_i^t \tag{Since $D_i^t = D_i^{t-1} +  d_i^t$}\\
    &\geq \gamma_i (D_i^t - (b_i - r_i^t) ) + \ell_i^t \tag{By \cref{eq:str4} and $\gamma_i \geq 0$)}\\
    &= \gamma_i D_i^t - \gamma_i(b_i - r_i^t) + k_i^{t+1} - r_i^t \tag{Since $k_i^{t+1} = r_i^t + \ell_i^t$}\\
    &= \gamma_i D_i^t + k_i^{t+1} - b_i + (1-\gamma_i)(b_i - r_i^t)\\
    &\geq \gamma_i D_i^t + k_i^{t+1} - b_i \tag{Since $\gamma_i \leq 1$ and $r_i^t < b_i$}\\
    &= \gamma_i D_i^t + (k_i^{t+1} - b_i)_+ \tag{By \cref{eq:str1}}
    \end{align*}
\end{enumerate}
\end{proof}
\begin{lemma}
\label{lem:helpful-manipulation}
If $0 < \tau \leq 1$, $w_i \leq p c_i$, and $\frac{\tau p c_i}{w_i} \leq \delta \leq 1$, then $(1 - \delta) (\frac{p c_i}{w_i} - 1) \leq \frac{(1 - \tau)^2}{4 \tau}$.
\end{lemma}
\begin{proof}
Let us define $y_i = \frac{p c_i}{w_i} \geq 1$.
Then,
\begin{align*}
(1 - \delta) \left( \frac{p c_i}{w_i} - 1 \right)
&= (1 - \delta) (y_i - 1) \tag{Since $y_i = \frac{p c_i}{w_i}$}\\
&\leq (1 - \tau y_i) (y_i - 1) \tag{Since $\delta \geq \frac{\tau p c_i}{w_i} = \tau y_i$ and $y_i \geq 1$}\\
&= -\tau \left( y_i - \frac{1 + \tau}{2 \tau} \right)^2 + \frac{(1 - \tau)^2}{4 \tau} \tag{Algebraic manipulation}\\
&\leq \frac{(1 - \tau)^2}{4 \tau} \tag{Since $\tau > 0$}
\end{align*}
\end{proof}

We are now ready to prove our structural lemma \cref{lem:structural}.
    
\structural*
\begin{proof}
We prove each property one by one.

\textbf{Property 1.}
Since $s^{\mathrm{end}} > 0$, all full shipments are fulfilled.
So,
\begin{align*}
\mathrm{transport}_i(\bar{\ell})
&= w_i \left\lceil \frac{L_i}{c_i} \right\rceil \tag{Since $\bar{\gamma}\mathrm{PA}$ always sends full shipments}\\
&\leq \frac{w_i}{c_i} (L_i + c_i)\\
&\leq \frac{w_i}{c_i} (\gamma_i D_i + b_i + 2c_i) \tag{By \cref{lem:upper-bound-on-cumulative}}\\
&\leq \gamma_i D_i \frac{w_i}{c_i} + p (b_i + 2 c_i) \tag{Since $w_i \leq p c_i$}
\end{align*}
Meanwhile, we see that
\begin{align*}
D_i^u(\bar{\ell})
&= D_i - D_i^m(\bar{\ell}) \tag{Since $D_i = D_i^m(\bar{\ell}) + D_i^u(\bar{\ell})$}\\
&= D_i - k_i^1 - L_i + s_i^{\mathrm{end}}(\bar{\ell}) \tag{Since $k_i^1 + L_i = D_i^m + s_i^{\mathrm{end}}(\bar{\ell})$}\\
&\leq D_i - k_i^1 - L_i + b_i + c_i \tag{By \cref{lem:bounded-terminal-inventory}}\\
&\leq D_i - k_i^1 - \gamma_i D_i + b_i + c_i \tag{By \cref{lem:lower-bound-on-cumulative-when-hub-not-exhausted}, since $s^{\mathrm{end}} > 0$}\\
&= (1 - \gamma_i) D_i + c_i \tag{Since $k_i^1 = b_i$}
\end{align*}
So, $\mathrm{penalty}_i(\bar{\ell}) = p D_i^u \leq (1-\gamma_i) p D_i + p c_i$ as desired.

\textbf{Property 2.}
Fix an arbitrary site $i \in [n]$.
\begin{align*}
&\; \mathrm{transport}_i(\bar{\ell}) - \mathrm{transport}_i(\OPT)\\
\leq &\; w_i \left\lceil \frac{L_i}{c_i} \right\rceil - w_i \left\lceil \frac{L_i^\star}{c_i} \right\rceil \tag{Since $\bar{\gamma}\mathrm{PA}$ always requests for full trucks}\\
\leq &\; w_i + \frac{w_i}{c_i} (L_i - L_i^\star)\\
\leq &\; w_i + \frac{w_i}{c_i} (\gamma_i D_i + b_i + c_i - L_i^\star) \tag{By \cref{lem:upper-bound-on-cumulative}}\\
\leq &\; 2 w_i + \frac{2 w_i b_i}{c_i} + \frac{w_i}{c_i} (\gamma_i D_i - b_i - L_i^\star) \tag{Pulling out $2 b_i + c_i$}\\
\leq &\; 2 w_i + \frac{2 w_i b_i}{c_i} + \frac{\gamma_i w_i}{c_i} (D_i - L_i^\star - b_i) \tag{Since $\gamma_i \leq 1$ and $L_i^\star, b_i \geq 0$}\\
\leq &\; 2 w_i + \frac{2 w_i b_i}{c_i} + \frac{\gamma_i w_i}{c_i} (D_i - L_i^\star - b_i)_+\\
= &\; 2 w_i + \frac{2 w_i b_i}{c_i} + \frac{\gamma_i w_i}{p c_i} \cdot \mathrm{penalty}_i(\OPT) \tag{By definition}\\
\leq &\; 2p (b_i + c_i) + \tau \cdot \mathrm{penalty}_i(\OPT) \tag{Since $w_i \leq p c_i$ and $\gamma_i \leq \frac{\tau p c_i}{w_i}$}
\end{align*}

\textbf{Property 3.}
We will upper bound $\mathrm{penalty}(\bar{\ell})$ and lower bound $\mathrm{penalty}(\OPT)$, before combining their implied inequalities.
\begin{align*}
\mathrm{penalty}(\bar{\ell})
&= \sum_{i \in [n]} p(D_i - D_i^m(\bar{\ell})) \tag{By definition}\\
&= \sum_{i \in [n]} p(D_i - L_i - k_i^1 + s_i^{\mathrm{end}}(\bar{\ell})) \tag{Since $D_i^m(\bar{\ell}) = L_i + k_i^1 - s_i^{\mathrm{end}}(\bar{\ell})$}\\
&\leq \sum_{i \in [n]} p(D_i - L_i - k_i^1 + b_i + c_i) \tag{By \cref{lem:bounded-terminal-inventory}}\\
&= \sum_{i \in [n]} p(D_i - L_i + c_i) \tag{Since $k_i^1 = b_i$}\\
&= \sum_{i \in [n]} p(D_i + c_i) - ps \tag{Since $s^{\mathrm{end}} = 0$ implies that $\sum_{i=1}^n L_i = s$}
\end{align*}
Meanwhile,
\begin{align*}
\mathrm{penalty}(\OPT)
&= \sum_{i \in [n]} p(D_i - L_i^\star - b_i) \tag{By definition}\\
&\geq \sum_{i \in [n]} p(D_i - b_i) - ps \tag{Since $\sum_{i=1}^n L_i^\star \leq s$}
\end{align*}
Putting together, we get $\mathrm{penalty}(\bar{\ell}) - \mathrm{penalty}(\OPT) \leq \sum_{i \in [n]} p(b_i + c_i)$ as desired.
\end{proof}

We are now ready to prove \cref{lem:generic-nonexhausted} and \cref{lem:generic-exhausted}.

\genericnonexhausted*
\begin{proof}
It suffices to bound this approximation for each site $i \in [n]$ since $\cost(\bar{\ell}) = \sum_{i=1}^n \cost_i(\bar{\ell})$.
In the rest of this proof, we will prove it with respect to an arbitrary fixed site $i \in [n]$.

We can lower bound $\cost(\OPT)$ in the relaxed setting where the only cost incurred is from transporting $L_i^\star = (D_i - b_i)_+$ supply to site $i$:
\begin{equation}
\label{eq:lower-bound-of-OPT}
\cost_i(\OPT)
\geq \mathrm{transport}_i(\OPT)
= w_i \cdot \left\lceil \frac{(D_i - b_i)_+}{c_i} \right\rceil
\geq (D_i - b_i) \frac{w_i}{c_i}
\end{equation}

Observe that
\begin{align*}
\cost_i(\bar{\ell})
&= \mathrm{transport}_i(\bar{\ell}) + \mathrm{penalty}_i(\bar{\ell})\\
&\leq \gamma_i D_i \frac{w_i}{c_i} + p(b_i + 2c_i) + (1-\gamma_i) p D_i + p c_i \tag{By \cref{eq:sitewise-transport-when-hub-not-exhausted} and \cref{eq:sitewise-penalty-when-hub-not-exhausted}}\\
&= \left( \gamma_i + (1 - \gamma_i) \frac{p c_i}{w_i} \right) \cdot \frac{w_i}{c_i} (D_i - b_i) + \gamma_i \frac{w_i}{c_i} b_i + (1 - \gamma_i) p b_i + p(b_i + 3c_i) \tag{Algebraic manipulation}\\
&\leq \left( \gamma_i + (1 - \gamma_i) \frac{p c_i}{w_i} \right) \cdot \frac{w_i}{c_i} (D_i - b_i) + p (2b_i + 3c_i) \tag{Since $w_i \leq p c_i$}\\
&\leq \left( \gamma_i + (1 - \gamma_i) \frac{p c_i}{w_i} \right) \cdot \cost_i(\OPT) + p (2b_i + 3c_i) \tag{By \cref{eq:lower-bound-of-OPT}}
\end{align*}

That is,
\begin{equation}
\label{eq:upper-bound-of-cost_i}
\cost_i(\bar{\ell})
\leq \left( \gamma_i + (1 - \gamma_i) \frac{p c_i}{w_i} \right) \cdot \cost_i(\OPT) + p (2b_i + 3c_i)
\end{equation}

We consider two cases:
(i) $\frac{\tau p c_i}{w_i} \geq 1$, and
(ii) $\frac{\tau p c_i}{w_i} < 1$.

In case (i), $\gamma_i \geq \min \left\{ 1, \frac{\tau p c_i}{w_i} \right\} = 1$.
Setting $\gamma_i = 1$ in \cref{eq:upper-bound-of-cost_i} yields $\cost_i(\bar{\ell}) \leq \cost_i(\OPT) + p (2b_i + 3c_i)$.
The claim follows because $\tau > 0$ implies that $\frac{(1 - \tau)^2}{4 \tau} \geq 0$, and so $1 \leq 1 + \frac{(1 - \tau)^2}{4 \tau}$.

In case (ii), we see that $\frac{\tau p c_i}{w_i} \leq \gamma_i \leq 1$.
Using \cref{lem:helpful-manipulation} with $\delta = \gamma_i$, we see that $\gamma_i + (1 - \gamma_i) \frac{p c_i}{w_i} = 1 + (1 - \gamma_i) \cdot ( \frac{p c_i}{w_i} - 1) \leq 1 +\frac{(1 - \tau)^2}{4 \tau}$.
Plugging this into \cref{eq:upper-bound-of-cost_i} yields the claim.
\end{proof}

\genericexhausted*
\begin{proof}
Let us define $\Delta(\mathrm{transport}) = \mathrm{transport}(\bar{\ell}) - \mathrm{transport}(\OPT)$ and $\Delta(\mathrm{penalty}) = \mathrm{penalty}(\bar{\ell}) - \mathrm{penalty}(\OPT)$.
Then,
\begin{align*}
\cost(\bar{\ell})
&= \cost(\OPT) + \Delta(\mathrm{transport}) + \Delta(\mathrm{penalty})\\
&\leq \cost(\OPT) + \tau \cdot \mathrm{penalty}(\OPT) + \sum_{i \in [n]} 2p (b_i + c_i) + \Delta(\mathrm{penalty}) \tag{By item 2 of \cref{lem:structural}}\\
&\leq \cost(\OPT) + \tau \cdot \mathrm{penalty}(\OPT) + \sum_{i \in [n]} 2p (b_i + c_i) + \sum_{i \in [n]} p(b_i + c_i) \tag{By item 3 of \cref{lem:structural}}\\
&\leq (1+\tau) \cdot \cost(\OPT) + \sum_{i \in [n]} 3 p (b_i + c_i) \tag{Since $\mathrm{penalty}(\OPT) \leq \cost(\OPT)$}
\end{align*}
\end{proof}

\cref{thm:online-upper-bound} follows immediately by setting $\tau = 1/3$.

\onlineupperbound*
\begin{proof}
Set $\tau = 1/3$, then apply \cref{lem:generic-nonexhausted} and \cref{lem:generic-exhausted}.
\end{proof}

\lowerone*
\begin{proof}
We will prove using Yao's minimax principle \cite{yao1977probabilistic}.
To do so, it suffices to construct a probability distribution $\mathcal{D}$ over \texttt{OSSA} instances such that for any deterministic online algorithm $\ALG$, the ratio of its expected cost over $\mathcal{D}$ to the optimal offline cost $\OPT$ is at least $\alpha$.

As a reminder, the superscripts in our notation are time steps and \emph{not} actual powers.
For instance, $d_i^2$ is the demand arriving at site $i$ at time step $2$.

Suppose, for contradiction, that $\alpha \leq k$ for some constant $k$.

\textbf{\texttt{OSSA} instance parameters.}
Let $n$ be even, $\gamma \geq 1$ be an arbitrary integer, and $s = \frac{n \gamma}{2}$.
Define the remaining \texttt{OSSA} instance parameters for each site $i \in [n]$ as follows:
\[
b_i = c_i = \gamma
\qquad
\text{and}
\qquad
w_i = \frac{p \gamma \eps}{2(k+1)}
\]
Under these parameters, we see that
\begin{equation}
\label{eq:lower1-additive-term}
(1 - \eps)\frac{\sum_{i \in [n]}(b_i + c_i)p}{8}
= (1 - \eps)\frac{np \gamma}{4}    
\end{equation}

\textbf{Demand arrival.}
There is only demand arriving at $t = 1$ and $t = 2$.
Define $d_i^1 = \gamma$ for all $i \in [n]$ so that the initial stock is completely consumed after the first time step at all sites $i \in [n]$.
We now define a distribution over the second-round demand.
Choose a subset $H \subseteq [n]$ uniformly at random among all subsets of size $n/2$, and set
\begin{equation}
\label{eq:lower1-second-round-demand}
d_i^2 = \begin{cases}
\gamma & i \in H\\
0 & i \notin H
\end{cases}
\end{equation}

\textbf{Lower bounding expected cost for any deterministic algorithm.}
Fix an arbitrary deterministic algorithm.
Let $k_i^2$ be the stock at site $i$ just \emph{before} the second-round demand arrives.
Then, over the random choice of $H$, we see that
\begin{align*}
\E_{H} \left[ \sum_{i=1}^n (d_i^2 - k_i^2) \right]
&= \E_{H} \left[ \sum_{i=1}^n \mathbbm{1}[i \in H] \cdot (\gamma - k_i^2) \right] \tag{By \cref{eq:lower1-second-round-demand}}\\
&= \sum_{i=1}^n \E_{H} \left[ \mathbbm{1}[i \in H] \cdot (\gamma - k_i^2) \right] \tag{By linearity of expectation}\\
&= \sum_{i=1}^n \Pr(i \in H) \cdot (\gamma - k_i^2)\\
&= \frac{1}{2} \sum_{i=1}^n (\gamma - k_i^2) \tag{Since $\Pr(i \in H) = 1/2$ for all $i \in [n]$}\\
&\geq \frac{n \gamma}{4} \tag{Since $\sum_{i=1}^n k_i^2 \leq s = \frac{n \gamma}{2}$}
\end{align*}
Hence, any deterministic online algorithm $\ALG$ has expected cost at least
\begin{equation}
\label{eq:lower1-ALG}
\E_{H}[\cost(\ALG)]
\geq \E_{H}[\mathrm{penalty}(\ALG)]
\geq \frac{np \gamma}{4}
\end{equation}

\textbf{Upper bounding cost of $\OPT$.}
On the other hand, the offline optimum $\OPT$ knows the realization of $H$ and can send $\gamma$ units to exactly the sites in $H$ at time $t=1$, so that $k_i^2 = \gamma$ if $i \in H$.
Thus
\begin{equation}
\label{eq:lower1-OPT}
\cost(\OPT)
= \mathrm{transport}(\OPT)
= w_i \cdot \frac{n}{2}
= \frac{np \gamma \eps}{4(k+1)}
\end{equation}

\textbf{Combining.}
By the assumption that $\ALG$ is $\alpha$-competitive with an additive error $\beta = \frac{(1 - \eps)p}{8} \sum (b_i + c_i)$, we rearrange the performance guarantee:
\begin{align*}
\alpha
&\geq \frac{\E[\cost(\ALG)] - \frac{(1 - \eps)p}{8} \cdot \sum_{i \in [n]} (b_i + c_i)}{\cost(\OPT)}\\
&\geq \frac{\frac{np \gamma}{4} - (1-\eps) \frac{np \gamma}{4}}{\frac{np \gamma \eps}{4(k+1)}} \tag{By \cref{eq:lower1-additive-term}, \cref{eq:lower1-ALG}, and \cref{eq:lower1-OPT}}\\
&= k+1
\end{align*}
Since this bound holds for any deterministic algorithm against the distribution $H$, Yao's minimax principle \cite{yao1977probabilistic} tells us that no randomized algorithm can achieve a competitive ratio less than $k+1$ with the given additive error.
This contradicts the assumption $\alpha \leq k$.
\end{proof}

\lowertwo*
\begin{proof}
We will prove using Yao's minimax principle \cite{yao1977probabilistic}.
To do so, it suffices to construct a probability distribution $\mathcal{D}$ over \texttt{OSSA} instances such that for any deterministic online algorithm $\ALG$ achieving $\cost(\ALG) \leq \alpha \cdot \cost(\OPT) + \beta$ with $\alpha = \left( \frac{4}{3} - \eps \right)$ must have $\beta \in \Omega(s p \eps)$.

As a reminder, the superscripts in our notation are time steps and \emph{not} actual powers.
For instance, $d_i^2$ is the demand arriving at site $i$ at time step $2$.

\textbf{\texttt{OSSA} instance parameters.}
Let $n = 2$, $w_1 = 0$, $w_2 = p/2$ and $b_1 = b_2 = c_1 = c_2 = 1$.

\textbf{Demand arrival.}
Define $d_1^1 = d_2^1 = 1$ so that the initial stock is completely consumed after the first time step at both sites.
We define a probability distribution over the two remaining demand arrival events $\mathcal{E}_1$ and $\mathcal{E}_2$ with equal probability.
\begin{enumerate}
    \item \textbf{$\mathcal{E}_1$:}
    Define $d_1^t = 0$ and $d_2^t = 1$ for $t = 2, \ldots, s$, and no further demand arrivals
    \item \textbf{$\mathcal{E}_2$:}
    Define $d_1^t = 0$ and $d_2^t = 1$ for $t = 2, \ldots, s$, and $d_1^t = 1$, and $d_2^t = 0$ for $t = s+1, \ldots, 2s$
\end{enumerate}

\textbf{Lower bounding expected cost for any deterministic algorithm.}
Fix an arbitrary deterministic algorithm $\ALG$.
Let $0 \leq m \leq s$ be the total number of demands at site $2$ served in time steps $t = 2, \ldots, s$.
This incurs a transport cost of $m w_2 = \frac{mp}{2}$.
Note that, under $\mathcal{E}_2$, there is total demand of $2s-1$, so at least $s-1$ demand will be unmet.
Thus,
\begin{align*}
\cost(\ALG \mid \mathcal{E}_1)
&= \mathrm{transport}(\ALG \mid \mathcal{E}_1) + \mathrm{penalty}(\ALG \mid \mathcal{E}_1)
= \frac{mp}{2} + p((s-1) - m)\\
\cost(\ALG \mid \mathcal{E}_2)
&= \mathrm{transport}(\ALG \mid \mathcal{E}_2) + \mathrm{penalty}(\ALG \mid \mathcal{E}_2)
\geq \frac{mp}{2} + p(s-1)
\end{align*}

\textbf{Upper bounding cost of $\OPT$.}
Meanwhile, $\OPT$ sees the event realization and can send $s-1$ supply to site 2 under $\mathcal{E}_1$ and send $s$ supply to site 1 under $\mathcal{E}_1$.
Thus,
\begin{align*}
\cost(\OPT \mid \mathcal{E}_1)
&= \mathrm{transport}(\OPT \mid \mathcal{E}_1) + \mathrm{penalty}(\OPT \mid \mathcal{E}_1)
= (s-1) w_2 + 0
= \frac{p(s-1)}{2}\\
\cost(\OPT \mid \mathcal{E}_2)
&= \mathrm{transport}(\OPT \mid \mathcal{E}_2) + \mathrm{penalty}(\OPT \mid \mathcal{E}_2)
= 0 + p(s-1)
\end{align*}

\textbf{Combining.}
Let $\alpha = \left( \frac{4}{3} - \eps \right)$.
Putting together the above, we see that
\begin{align*}
\cost(\ALG \mid \mathcal{E}_1) - \alpha \cdot \cost(\OPT \mid \mathcal{E}_1)
&\geq p \left( (s-1) - \frac{m}{2} - \frac{\alpha (s-1)}{2} \right)\\
\cost(\ALG \mid \mathcal{E}_2) - \alpha \cdot \cost(\OPT \mid \mathcal{E}_2)
&\geq p \left( (s-1) + \frac{m}{2} - \alpha (s-1) \right)
\end{align*}
So, over in expectation over the events,
\begin{align*}
&\; \E \left[ \cost(\ALG) - \alpha \cdot \cost(\OPT) \right]\\
\geq &\; \frac{p}{2} \left( (s-1) - \frac{m}{2} - \frac{\alpha (s-1)}{2} \right) + \frac{p}{2} \left( (s-1) + \frac{m}{2} - \alpha (s-1) \right) \tag{Since $\Pr(\mathcal{E}_1) = \Pr(\mathcal{E}_2) = 1/2$}\\
= &\; p \left( (s-1) - \frac{3 \alpha (s-1)}{4} \right)\\
= &\; \frac{3 (s-1) p \eps}{4} \tag{Since $\alpha = \left( \frac{4}{3} - \eps \right)$}
\end{align*}
Since this bound holds for any deterministic algorithm against the given demand arrival distribution above, Yao's minimax principle \cite{yao1977probabilistic} tells us that any randomized algorithm achieving $\cost(\ALG) \leq \alpha \cdot \cost(\OPT) + \beta$ with $\alpha = \left( \frac{4}{3} - \eps \right)$ must have $\beta \in \Omega(s p \eps)$.
\end{proof}

\subsection{Proofs for \texorpdfstring{\cref{sec:learning-augmented}}{Section 6}}

\advicetooweak*
\begin{proof}
Suppose, for a contradiction, that there is an online algorithm \textsc{ALG} and a constant $C > 0$ such that $\E[\cost(\textsc{ALG})] \leq \cost(\textsc{OPT}) + C \sum_{i=1}^n p(b_i+c_i)$ on every instance.

We will prove using Yao's minimax principle \cite{yao1977probabilistic}.
To do so, it suffices to construct a probability distribution $\mathcal{D}$ over \texttt{OSSA} instances such that for any deterministic online algorithm $\cost(\textsc{ALG}) > \cost(\textsc{OPT}) + C \sum_{i=1}^n p(b_i+c_i)$.
We will construct a distribution $\mathcal{D}$ over \texttt{OSSA} instances that induce identical predictions (for each prediction type), so even with perfect predictions no algorithm can distinguish among instances drawn from $\mathcal{D}$.

As a reminder, the superscripts in our notation are time steps and \emph{not} actual powers.
For instance, $d_i^2$ is the demand arriving at site $i$ at time step $2$.

\textbf{\texttt{OSSA} instance parameters.}
Let $n = 3$, $w_1 = 0$, $w_2 = p/2$, $w_3= p$ and $b_i = c_i = 1$ for all $i \in [n]$.

\textbf{Demand arrival.}
Define $d_1^1 = d_2^1 = d_3^1 = 1$ so that the initial stock is completely consumed after the first time step at all three sites. 
Then, we define a probability distribution that assigns probability $1/2$ to each of the two remaining demand-arrival events, $\mathcal{E}_1$ and $\mathcal{E}_2$.

We use this construction in two separate cases: first, for site-wise total demand prediction and one-step lookahead prediction; and second, for total supply and total demand prediction. 

\begin{enumerate}
    \item We first tackle the site-wise total demand prediction and one-step lookahead prediction.
    Define constant $K = 28C$.
    For both $\mathcal{E}_1$ and $\mathcal{E}_2$, we have $d_1^t = 0, d_2^t = 1,d_3^t=0$ for timestep $t = 2,\dots, K+1$  and we have $d_1^{t} = 1, d_2^t = 0,d_3^t=0$ for timestep  $t = K+2,\dots, 2K+1$. As the demand arrivals for both events \(\mathcal{E}_1\) and \(\mathcal{E}_2\) are the same, the predictions are unable to inform the online algorithm.  The difference in \(\mathcal{E}_1\) and \(\mathcal{E}_2\) is that $s = k$ under \(\mathcal{E}_1\) and $s = 2k$  \(\mathcal{E}_2\). 

\textbf{Lower bounding expected cost for any deterministic algorithm.}
Fix an arbitrary deterministic algorithm $\ALG$.
Let $0 \leq m \leq  K$ be the total number of demands at site $2$ served in time steps $t = 2, \ldots, K+1$.
This incurs a transport cost of $m w_2 = \frac{mp}{2}$. In \(\mathcal{E}_1\), as $D_1+D_2+D_3 - k_1^1-k_2^1-k_3^1= 2K =s+K$,  hence $\mathrm{penalty}(\ALG\mid \mathcal{E}_1) \geq Kp$. In \(\mathcal{E}_2\) as  only $m$ of the total demand at site $2$ was serve, this incur a penalty of $\mathrm{penalty}(\ALG\mid \mathcal{E}_2) \geq (K-m)p$
\begin{align*}
\cost(\ALG \mid \mathcal{E}_1)
&= \mathrm{transport}(\ALG \mid \mathcal{E}_1) + \mathrm{penalty}(\ALG \mid \mathcal{E}_1)
\geq \frac{mp}{2} + Kp\\
\cost(\ALG \mid \mathcal{E}_2)
&= \mathrm{transport}(\ALG \mid \mathcal{E}_2) + \mathrm{penalty}(\ALG \mid \mathcal{E}_2)
\geq \frac{mp}{2} + (K-m)p
\end{align*}

\textbf{Upper bounding cost of $\OPT$.}
Meanwhile $\OPT$ see the total supply and send no supply to site 2 in $\mathcal{E}_1$ and send $K$ supplies to both site 1 and 2 in $\mathcal{E}_2$.
    \begin{align*}
\cost(\OPT \mid \mathcal{E}_1)
&= \mathrm{transport}(\OPT \mid \mathcal{E}_1) + \mathrm{penalty}(\OPT \mid \mathcal{E}_1)
= Kp\\
\cost(\OPT \mid \mathcal{E}_2)
&= \mathrm{transport}(\OPT \mid \mathcal{E}_2) + \mathrm{penalty}(\OPT \mid \mathcal{E}_2)
= \frac{Kp}{2}
\end{align*}
 \textbf{Combining}.
\begin{align*}
\; \E \left[ \cost(\ALG) - \alpha \cdot \cost(\OPT) \right]
&= \frac{1}{2}\left(\cost(\ALG \mid \mathcal{E}_1) + \cost(\ALG \mid \mathcal{E}_2) - \cost(\OPT \mid \mathcal{E}_1) - \cost(\OPT \mid \mathcal{E}_2) \right)\\
 &\geq 
 \frac{Kp}{4} = 7Cp>6Cp = C \sum_{i =1}^3 (b_i+c_i)p\end{align*}

 \item  We next tackle the total demand and total supply prediction. Let $K = 28 C$. For both \(\mathcal{E}_1\) and \(\mathcal{E}_2\), we have $s = K$. We define a probability distribution over the two remaining demand arrival events $\mathcal{E}_1$ and $\mathcal{E}_2$ with equal probability. 
 
 \begin{enumerate} 
 \item \textbf{$\mathcal{E}_1$:} Define $d_1^t = 0, d_2^t = 1,d_3^t=0$ for timestep $t = 2,\dots, K+1$ and we have $d_1^{t} = 1, d_2^t = 0,d_3^t=0$ for timestep $t = K+2,\dots, 2K+1$. 
 \item \textbf{$\mathcal{E}_2$:} Define $d_1^t = 0, d_2^t = 1,d_3^t=0$ for timestep $t = 2,\dots, K+1$ and we have $d_1^{t} = 0, d_2^t = 0,d_3^t=1$ for timestep $t = K+2,\dots, 2K+1$. 
 \end{enumerate} 
 
 We note that $D_1+D_2+D_3 = 2K+3$ under both \(\mathcal{E}_1\) and \(\mathcal{E}_2\). Hence these predictions are
unable to inform the online algorithm which event will occur.

\textbf{Lower bounding expected cost for any deterministic algorithm.}
Fix an arbitrary deterministic algorithm \(\ALG\). Let \(0\le m\le K\) be the
total number of demands at site \(2\) served in time steps
\(t=2,\ldots,K+1\). This incurs transport cost
\(
mw_2=\frac{mp}{2}.
\) In \(\mathcal{E}_1\), the later demand is at site \(1\). Since site \(1\) has
transport cost zero, serving site \(2\) early does not reduce the total penalty
relative to saving the supply for site \(1\), but it does incur transport cost
\(\frac{mp}{2}\). Since the total demand exceeds supply by \(K\), we have
\(
\mathrm{penalty}(\ALG\mid \mathcal{E}_1)\ge Kp.
\) Therefore,
\[
\cost(\ALG\mid \mathcal{E}_1)
\ge
\frac{mp}{2}+Kp.
\]

In \(\mathcal{E}_2\), the later demand is at site \(3\). Since \(\ALG\) served
only \(m\) units of the site-\(2\) demand during the first block, at least
\(K-m\) units of demand at site \(2\) remain unmet. Also, the final \(K\) units
of demand at site \(3\) cost at least \(Kp\), either as penalty or as transport
cost. Hence
\[
\cost(\ALG\mid \mathcal{E}_2)
\ge
\frac{mp}{2}+(K-m)p+Kp.
\]

\textbf{Upper bounding cost of \(\OPT\).}
Meanwhile, \(\OPT\) sees the full demand sequence in advance. In
\(\mathcal{E}_1\), \(\OPT\) sends the available extra supply to site \(1\),
which has transport cost zero, and pays penalty for the site-\(2\) demand.
Thus
\[
\cost(\OPT\mid \mathcal{E}_1)=Kp.
\]

In \(\mathcal{E}_2\), \(\OPT\) sends the available extra supply to site \(2\),
because site \(2\) has transport cost \(w_2=p/2\), while site \(3\) has
transport cost \(w_3=p\). It then pays penalty for the site-\(3\) demand.
Therefore,
\[
\cost(\OPT\mid \mathcal{E}_2)
=
Kw_2+Kp
=
\frac{Kp}{2}+Kp
=
\frac{3Kp}{2}.
\]

\textbf{Combining.}
\begin{align*}
\E\left[\cost(\ALG)-\cost(\OPT)\right]
&=
\frac12\left(
\cost(\ALG\mid \mathcal{E}_1)-\cost(\OPT\mid \mathcal{E}_1)
\right) \\
&\qquad
+
\frac12\left(
\cost(\ALG\mid \mathcal{E}_2)-\cost(\OPT\mid \mathcal{E}_2)
\right) \\
&\ge
\frac12\cdot \frac{mp}{2}
+
\frac12\left(
\frac{mp}{2}+(K-m)p+Kp-\frac{3Kp}{2}
\right) \\
&=
\frac12\cdot \frac{mp}{2}
+
\frac12\cdot \frac{(K-m)p}{2} \\
&=
\frac{Kp}{4}
=
7Cp
>
6Cp
=
C\sum_{i=1}^3 (b_i+c_i)p.
\end{align*}

\end{enumerate}

\end{proof}

For \cref{thm:learning-augmented}, let $\ALG$ be $\bar{\gamma}\mathrm{PA}$ run on the $\bar{\gamma}$ produced by \cref{alg:prediction} for some $\lambda \in (0, \frac{1}{3})$.
We will prove each part separately and rely on the following helper lemmas.

\begin{lemma}
\label{lem:tau-lambda-relationship}
For $0 < \lambda \leq \frac{1}{3}$.
If $\tau = (\sqrt{1 + \lambda} - \sqrt{\lambda})^2$, then $\frac{(1 - \tau)^2}{4 \tau} = \lambda$ and $\tau \leq \frac{(1 - \lambda)^2}{4 \lambda}$.
\end{lemma}
\begin{proof}

    For the first claim, we first note that
\[
\sqrt{1+\lambda}-\sqrt{\lambda}
= \frac{(\sqrt{1+\lambda}-\sqrt{\lambda})
(\sqrt{1+\lambda}+\sqrt{\lambda})}
{\sqrt{1+\lambda}+\sqrt{\lambda}} 
= \frac{(1+\lambda)-\lambda}
{\sqrt{1+\lambda}+\sqrt{\lambda}} 
= \frac{1}{\sqrt{1+\lambda}+\sqrt{\lambda}}.
\]

Hence $\left(\sqrt{1+\lambda}-\sqrt{\lambda}\right)^2= \frac{1}{\left(\sqrt{1+\lambda}+\sqrt{\lambda}\right)^2}$. Thus,

\[\frac{(1-\tau)^2}{4\tau}
= \frac{1}{4}\left(\frac{1}{\sqrt{\tau}}-\sqrt{\tau}\right)^2 
= \frac{1}{4}\left(\sqrt{1+\lambda}+\sqrt{\lambda}
      - \sqrt{1+\lambda}+\sqrt{\lambda}\right)^2 
= \lambda \]

For the second claim, as we showed that
\(
\lambda = \frac{(1-\tau)^2}{4\tau}
\), it is enough to show that \(\lambda \leq \tau\), because then
\(1-\tau \leq 1-\lambda\), and hence
\(
4\lambda\tau = (1-\tau)^2 \leq (1-\lambda)^2.
\)
Dividing by \(4\lambda>0\) gives the desired result.

It remains to prove \(\lambda \leq \tau\). Since \(\lambda \leq \frac13\), we have
\(
4\lambda \leq 1+\lambda
\) and therefore
\(
2\sqrt{\lambda} \leq \sqrt{1+\lambda}
\). 
Rearranging, we have 
\(
\sqrt{\lambda}
\leq
\sqrt{1+\lambda}-\sqrt{\lambda}.
\)
Squaring both sides gives
\(
\lambda \leq \left(\sqrt{1+\lambda}-\sqrt{\lambda}\right)^2 = \tau
\) and
hence
\(
\tau \leq \frac{(1-\lambda)^2}{4\lambda}.
\)
\end{proof}

Before we proceed, let us recall some notation from the main paper that will be helpful.
Define the pivotal index $i^\star = \min \{ i \in [n] : \sum_{j=1}^i N_j \geq s \}$ and the corresponding pivotal value $\zeta = (s - \sum_{j=1}^{i^\star-1} N_j) / N_{i^\star} \in (0,1]$.
If $s \geq \sum_{j=1}^n N_j$, we define $i^\star = n$ and $\zeta = 1$.
Let us now write the optimal allocation and allocation fractions respectively as
\[
L_i^\star =
(\ell_i^1)^\star =
\begin{cases}
N_i & i < i^\star\\
\zeta \cdot N_{i^\star} & i = i^\star\\
0 & i > i^\star
\end{cases}
\qquad
\text{and}
\qquad
\gamma_i^\star =
\begin{cases}
1 & i < i^\star\\
\zeta, & i = i^\star\\
0 & i > i^\star
\end{cases}
\]
Given predictions $\hat{s}$ and $\{ \hat{D}_i \}_{i \in [n]}$, let $\hat{i}^\star$, $\hat{\zeta}$, $\hat{L}_i$ and $\hat{\gamma}_i$ be corresponding terms computed using the predictions.
Furthermore, in \cref{alg:prediction}, we have
\[
\gamma_i
= \begin{cases}
\texttt{upper}_i & \text{if $i \in \{1, \ldots, i^\star - 1\}$}\\
\min\{\texttt{upper}_i, \max\{\hat{\zeta}, \texttt{lower}_i \}\} & \text{if $i = i^\star$}\\
\texttt{lower}_i & \text{if $i \in \{i^\star + 1, \ldots, n\}$}
\end{cases}
\]
where $\tau = ( \sqrt{1 + \lambda} -\sqrt{\lambda} )^2$, $\texttt{lower}_i = \min \{ 1, \frac{\lambda p c_i}{w_i} \}$, $\texttt{upper}_i = \min \{ 1, \frac{\tau p c_i}{w_i} \}$.

\begin{lemma}
\label{lem:smoothness-helper1}
Consider \cref{alg:prediction}.
Recall the definitions of $\hat{L}_i^\star$ and $\hat{\gamma}_i$ above.
Define
\[
i' = \begin{cases}
n & \text{if $\sum_{i=1}^n \hat{D}_i < \hat{s}$}\\
\hat{i}^\star - 1 & \text{if $\gamma_{\hat{i}^\star} < \min\{1, \frac{\tau p c_{\hat{i}^\star}}{w_{\hat{i}^\star}} \}$}\\
\hat{i}^\star & \text{otherwise}
\end{cases}
\]
Then, for sites $i \in \{i' + 1, \ldots, n\}$, we have $\gamma_i \geq \hat{\gamma}_i$.
\end{lemma}
\begin{proof}
For $i > \hat{i}^\star$, we see that $\hat{\gamma}_i = 0 \leq \min \{1, \frac{\lambda p c_i}{w_i} \} = \gamma_i$.
Meanwhile, $i^\star \in \{i'+1, \ldots, n\}$ only if $\gamma_{\hat{i}^\star} < \min\{1, \frac{\tau p c_{\hat{i}^\star}}{w_{\hat{i}^\star}} \} = \texttt{upper}_i$.
Since $\gamma_{\hat{i}^\star} = \min\{\texttt{upper}_i, \max\{\hat{\zeta}, \texttt{lower}_i \}\}$, it must be the case that $\gamma_{\hat{i}^\star} = \max\{\hat{\zeta}, \texttt{lower}_i \} > \hat{\zeta} = \hat{\gamma}_{\hat{i}^\star}$.
\end{proof}

\begin{lemma}
\label{lem:smoothness-helper2}
Consider \cref{alg:prediction}.
Recall the definitions of $\hat{L}_i^\star$ and $\hat{\gamma}_i$ above.
Define
\[
i' = \begin{cases}
n & \text{if $\sum_{i=1}^n \hat{D}_i < \hat{s}$}\\
\hat{i}^\star - 1 & \text{if $\gamma_{\hat{i}^\star} < \min\{1, \frac{\lambda p c_{\hat{i}^\star}}{w_{\hat{i}^\star}} \}$}\\
\hat{i}^\star & \text{otherwise}
\end{cases}
\]
Then, for sites $i \in \{1, \ldots, i'\}$, we have $\gamma_i \leq \hat{\gamma}_i$.
\end{lemma}
\begin{proof}
For $i > \hat{i}^\star$, we established that $\hat{\gamma_{i}} = 1 \geq \gamma_i$,  as $\gamma_i \in [0,1]$.
Meanwhile, $\hat{i}^\star \in \{1, \dots, i'\}$ only if $\gamma_{\hat{i}^\star} \geq \min \{1, \frac{\lambda p c_{\hat{i}^\star}}{w_{\hat{i}^\star}} \} = \texttt{lower}_i$.
Since $\gamma_{\hat{i}^\star} = \min\{\texttt{upper}_i, \max\{\hat{\zeta}, \texttt{lower}_i \}\}$, it must be the case that $\gamma_{\hat{i}^\star} \geq \hat{\zeta}$ as well, so $\gamma_{\hat{i}^\star} \geq \hat{\zeta} = \hat{\gamma}_{\hat{i}^\star}$.
\end{proof}

\begin{lemma}
\label{lem:smoothness-helper3}
For the predicted instance \((\hat{s},\{\hat{D}_i\}_{i\in[n]})\), define
\(\hat{N}_i=(\hat{D}_i-k_i^1)_+\) and
\[
\hat{L_i^\star}
=
\min\left\{
\hat{N}_i,
\left(\hat{s}-\sum_{j=1}^{i-1}\hat{N}_j\right)_+
\right\}.
\]
Let \(\hat{i}^\star=\max\{i\in[n]:\hat{L_i^\star}>0\}\), and define
\[
\hat{\zeta}
=
\begin{cases}
\frac{\hat{L}^\star_{\hat{i}^\star}}{\hat{N}_{\hat{i}^\star}}
& \text{if } \hat{N}_{\hat{i}^\star}>0,\\
1
& \text{if } \hat{N}_{\hat{i}^\star}\leq 0.
\end{cases}
\]
Let
\[
\hat{\gamma}_i
=
\begin{cases}
1 & \text{if } i<\hat{i}^\star,\\
\hat{\zeta} & \text{if } i=\hat{i}^\star,\\
0 & \text{if } i>\hat{i}^\star.
\end{cases}
\]

For the true instance \((s,\{D_i\}_{i\in[n]})\), define
\(N_i=(D_i-k_i^1)_+\) and
\[
L_i^\star
=
\min\left\{
N_i,
\left(s-\sum_{j=1}^{i-1}N_j\right)_+
\right\}.
\]
Let \(i^\star=\max\{i\in[n]:L_i^\star>0\}\), and define
\[
\zeta
=
\begin{cases}
\frac{L_{i^\star}}{N_{i^\star}}
& \text{if } N_{i^\star}>0,\\
1
& \text{if } N_{i^\star}\leq 0.
\end{cases}
\]
Let
\[
\gamma_i^{\textsc{OPT}}
=
\begin{cases}
1 & \text{if } i<i^\star,\\
\zeta & \text{if } i=i^\star,\\
0 & \text{if } i>i^\star.
\end{cases}
\]
Then
\[
\sum_{i\in[n]} N_i\left|\gamma_i^{\textsc{OPT}}-\hat{\gamma}_i\right|
\leq
3\operatorname{err}(\hat{s},\hat{D}),
\]

\end{lemma}

\begin{proof}
By construction, \(L_i^\star=\gamma_i^{\textsc{OPT}}N_i\) and
\(\hat{L_i^\star}=\hat{\gamma}_i\hat{N}_i\). We prove the bound through three
intermediate lemmas.

\begin{lemma}
\label{lem:err1}
We have
\[
\sum_{i\in[n]}N_i\left|\gamma_i^{\textsc{OPT}}-\hat{\gamma}_i\right|
\leq
\sum_{i\in[n]}|L_i^\star-\hat{L_i^\star}|
+
\sum_{i\in[n]}|N_i-\hat{N}_i|.
\]
\end{lemma}

\begin{proof}
For every site \(i\),
\[
N_i\left|\gamma_i^{\textsc{OPT}}-\hat{\gamma}_i\right|
=
\left|\gamma_i^{\textsc{OPT}}N_i-\hat{\gamma}_iN_i\right|
=
\left|L_i^\star-\hat{\gamma}_iN_i\right|.
\]
By the triangle inequality,
\[
\left|L_i^\star-\hat{\gamma}_iN_i\right|
\leq
|L_i^\star-\hat{L_i^\star}|
+
|\hat{L_i^\star}-\hat{\gamma}_iN_i|.
\]
Since \(\hat{L_i^\star}=\hat{\gamma}_i\hat{N}_i\), we have
\[
|\hat{L_i^\star}-\hat{\gamma}_iN_i|
=
\hat{\gamma}_i|\hat{N}_i-N_i|
\leq
|\hat{N}_i-N_i|,
\]
where the last step uses \(0\leq \hat{\gamma}_i\leq 1\). Therefore,
\[
N_i\left|\gamma_i^{\textsc{OPT}}-\hat{\gamma}_i\right|
\leq
|L_i^\star-\hat{L_i^\star}|
+
|N_i-\hat{N}_i|.
\]
Summing over all sites proves the claim.
\end{proof}

Now define the intermediate allocation $Y$ by
\[
Y_i
=
\min\left\{
N_i,
\left(\hat{s}-\sum_{h<i}N_h\right)_+
\right\}.
\]
This is the allocation obtained by using the true net demands \(N_i\), but the
predicted supply \(\hat{s}\).

\begin{lemma}
\label{lem:err2}
We have
\[
\sum_{i\in[n]}|L_i^\star-Y_i|
\leq
|s-\hat{s}|.
\]
\end{lemma}

\begin{proof}
The only difference between \(L_i^\star\) and $Y$ is the amount of available supply.
Both allocations use the same net demands and serve sites in the same fixed
order. Hence one allocation is contained in the other: if \(\hat{s}\geq s\),
then \(Y_i\geq L_i^\star\) for every \(i\), and if \(\hat{s}\leq s\), then
\(Y_i\leq L_i^\star\) for every \(i\). Therefore,
\[
\sum_{i\in[n]}|L_i^\star-Y_i|
=
\left|\sum_{i\in[n]}L_i^\star-\sum_{i\in[n]}Y_i\right|.
\]
The total amount served by \(x\) is
\(\sum_i L_i^\star=\min\{s,\sum_i N_i\}\), and the total amount served by $Y$ is
\(\sum_i Y_i=\min\{\hat{s},\sum_i N_i\}\). Thus
\[
\sum_{i\in[n]}|L_i^\star-Y_i|
=
\left|
\min\left\{s,\sum_{i\in[n]}N_i\right\}
-
\min\left\{\hat{s},\sum_{i\in[n]}N_i\right\}
\right|
\leq
|s-\hat{s}|.
\]
The final inequality follows from
\(|\min\{a,b\}-\min\{a,c\}|\leq |b-c|\).
\end{proof}

\begin{lemma}
\label{lem:err3}
We have
\[
\sum_{i\in[n]}|Y_i-\hat{L_i^\star}|
\leq
2\sum_{i\in[n]}|N_i-\hat{N}_i|.
\]
\end{lemma}

\begin{proof}
Define intermediate net-demand vectors \(N^{(0)}=N\), and for
\(k=1,\ldots,n\),
\[
N^{(k)}
=
(\hat{N}_1,\ldots,\hat{N}_k,N_{k+1},\ldots,N_n).
\]
Thus \(N^{(n)}=\hat{N}\). Let \(z^{(k)}\) be the prefix allocation with supply
\(\hat{s}\) and net-demand vector \(N^{(k)}\). Hence \(z^{(0)}=Y\) and
\(z^{(n)}=\hat{{L_i^\star}}\).

When moving from \(N^{(k-1)}\) to \(N^{(k)}\), only site \(k\)'s net demand
changes. Let \(\delta_k=\hat{N}_k-N_k\), so
\(|\delta_k|=|\hat{N}_k-N_k|\). Changing site \(k\)'s net demand by
\(|\delta_k|\) can change the allocation at site \(k\) by at most
\(|\delta_k|\). It can also change the total amount of supply available to
later sites by at most \(|\delta_k|\). Therefore,
\[
\sum_{i\in[n]}|z_i^{(k)}-z_i^{(k-1)}|
\leq
2|\delta_k|
=
2|\hat{N}_k-N_k|.
\]
Using the triangle inequality over the sequence
\(z^{(0)},z^{(1)},\ldots,z^{(n)}\), we get
\[
\sum_{i\in[n]}|Y_i-\hat{L_i^\star}|
\leq
\sum_{k=1}^n\sum_{i\in[n]}|z_i^{(k)}-z_i^{(k-1)}|
\leq
2\sum_{k=1}^n|N_k-\hat{N}_k|.
\]
\end{proof}

Combining the three lemmas,
\begin{align*}
\sum_{i\in[n]}N_i\left|\gamma_i^{\textsc{OPT}}-\hat{\gamma}_i\right|
&\leq
\sum_{i\in[n]}|L_i^\star-\hat{L_i^\star}|
+
\sum_{i\in[n]}|N_i-\hat{N}_i| \\
&\leq
\sum_{i\in[n]}|L_i^\star-Y_i|
+
\sum_{i\in[n]}|Y_i-\hat{L_i^\star}|
+
\sum_{i\in[n]}|N_i-\hat{N}_i| \\
&\leq
|s-\hat{s}|
+
3\sum_{i\in[n]}|N_i-\hat{N}_i|.
\end{align*}
Finally, because \(N_i=(D_i-k_i^1)_+\) and
\(\hat{N}_i=(\hat{D}_i-k_i^1)_+\), the map \(x\mapsto (x-k_i^1)_+\) is
1-Lipschitz, so \(|N_i-\hat{N}_i|\leq |D_i-\hat{D}_i|\). Hence
\[
\sum_{i\in[n]}N_i\left|\gamma_i^{\textsc{OPT}}-\hat{\gamma}_i\right|
\leq
|s-\hat{s}|
+
3\sum_{i\in[n]}|D_i-\hat{D}_i|
\leq
3\operatorname{err}(\hat{s},\hat{D}),
\]
as required.
\end{proof}

We are now ready to prove \cref{thm:learning-augmented}, from follows by combining \cref{lem:learning-augmented-robustness} and \cref{lem:learning-augmented-smoothness}.

\begin{lemma}
\label{lem:learning-augmented-robustness}
Consider the predictions $\{\hat{s}$, $\{ \hat{D}_i \}_{i \in [n]}\}$ for the given \texttt{OSSA} instance with prediction error $\eta = | s - \hat{s} | + \sum_{i=1}^n | D_i - \hat{D}_i | \geq 0$.
Let $\ALG$ be $\bar{\gamma}\mathrm{PA}$ run on the $\bar{\gamma}$ produced by \cref{alg:prediction} for some $\lambda \in (0, \frac{1}{3})$.
Then, $\cost(\ALG) \leq (1 + \frac{(1 - \lambda)^2}{4 \lambda}) \cdot \cost(\OPT) + O(\sum_{i \in [n]} (b_i + c_i))$.
\end{lemma}
\begin{proof}
By construction, we have
$
\min \left\{ 1, \frac{\lambda p c_i}{w_i} \right\}
\leq \gamma_i
\leq \left\{ 1, \frac{\tau p c_i}{w_i} \right\}
$
for all $i \in [n]$.

We consider the cases of $s^{\mathrm{end}} > 0$ and $s^{\mathrm{end}} = 0$ separately.

Suppose $s^{\mathrm{end}} > 0$.
Since $\gamma_i \geq \min \left\{ 1, \frac{\lambda p c_i}{w_i} \right\}$ and $\lambda \in (0,1]$, \cref{lem:generic-nonexhausted} tells us that
\[
\cost(\bar{\ell})
\leq \left( 1 + \frac{(1-\lambda)^2}{4\lambda} \right) \cdot \cost(\OPT) + \sum_{i \in [n]} 3p (b_i + c_i)
\]
 
Suppose $s^{\mathrm{end}} = 0$.
We have
\begin{align*}
\cost(\bar{\ell})
&\leq (1+\tau) \cdot \cost(\OPT) + \sum_{i \in [n]} 3 p (b_i + c_i) \tag{By \cref{lem:generic-exhausted} since $\gamma_i \leq \min \left\{ 1, \frac{\tau p c_i}{w_i} \right\}$}\\
&\leq \left( 1 + \frac{(1-\lambda)^2}{4\lambda} \right) \cdot \cost(\OPT) + \sum_{i \in [n]} 3 p (b_i + c_i) \tag{By \cref{lem:tau-lambda-relationship}}
\end{align*}
\end{proof}

\begin{lemma}
\label{lem:learning-augmented-smoothness}
Consider the predictions $\{\hat{s}$, $\{ \hat{D}_i \}_{i \in [n]}\}$ for the given \texttt{OSSA} instance with prediction error $\eta = | s - \hat{s} | + \sum_{i=1}^n | D_i - \hat{D}_i | \geq 0$.
Let $\ALG$ be $\bar{\gamma}\mathrm{PA}$ run on the $\bar{\gamma}$ produced by \cref{alg:prediction} for some $\lambda \in (0, \frac{1}{3})$.
Then, $\cost(\ALG) \leq (1 + \lambda) \cdot \cost(\OPT) + 3 \eta p + O( \sum_{i \in [n]} p(b_i + c_i))$.
\end{lemma}
\begin{proof}
We consider the cases of $s^{\mathrm{end}} > 0$ and $s^{\mathrm{end}} = 0$ separately.
Let $\hat{L}_i^\star = \min\{\hat{D}_i, \hat{s} - \sum_{j=1}^{i-1} \hat{D}_j\}$ be the optimal allocation under predictions $\hat{s}$ and $\{ \hat{D}_i \}_{i \in [n]}$, and $\hat{\gamma}_i = \hat{L}_i^\star / \hat{D}_i \in [0,1]$ be the corresponding ratio.
That is, for $i \in [n]$, we have
\[
\hat{\gamma}_i
= \begin{cases}
1 & \text{if $i < \hat{i}^\star$}\\
\hat{\zeta} & \text{if $i = \hat{i}^\star$}\\
0 & \text{if $i > \hat{i}^\star$}
\end{cases}
\]

Suppose $s^{\mathrm{end}} > 0$.
Define
\[
i' = \begin{cases}
n & \text{if $\sum_{i=1}^n \hat{D}_i < \hat{s}$}\\
\hat{i}^\star - 1 & \text{if $\gamma_{\hat{i}^\star} < \min\{1, \frac{\tau p c_{\hat{i}^\star}}{w_{\hat{i}^\star}} \}$}\\
\hat{i}^\star & \text{otherwise}
\end{cases}
\]
\begin{enumerate}
    \item For $i \in \{1, \ldots, i'\}$, we have $\gamma_i = \min\{1, \frac{\tau p c_i}{w_i}\} < \frac{\tau p c_i}{w_i}$ by construction.
So, for $i \in \{1, \ldots, i'\}$,
\begin{align*}
\cost_i(\bar{\ell})
&\leq \left( 1 + \frac{(1-\tau)^2}{4\tau} \right) \cdot \cost_i(\OPT) + \sum_{i \in [n]} 3p (b_i + c_i) \tag{By site-wise proof in \cref{lem:generic-nonexhausted}}\\
&= (1 + \lambda) \cdot \cost_i(\OPT) + \sum_{i \in [n]} 3p (b_i + c_i) \tag{By \cref{lem:tau-lambda-relationship}}
\end{align*}
\item For $i \in \{i' + 1, \ldots, n\}$, we have $\gamma_i \geq \hat{\gamma}_i$ by \cref{lem:smoothness-helper1}.

Meanwhile, by \cref{lem:smoothness-helper1}, we have $\gamma_i \geq \hat{\gamma}_i$ for $i \in \{i' + 1, \ldots, n\}$.
So
\begin{align*}
\cost_i(\bar{\ell})
&= \mathrm{transport}_i(\bar{\ell}) + \mathrm{penalty}_i(\bar{\ell})\\
&\leq \gamma_i D_i \frac{w_i}{c_i} + p(b_i + 2c_i) + (1-\gamma_i) p D_i + p (b_i + c_i) \tag{By \cref{eq:sitewise-transport-when-hub-not-exhausted} and \cref{eq:sitewise-penalty-when-hub-not-exhausted}}\\
&= p D_i - \gamma_i D_i (p - \frac{w_i}{c_i}) + p (2b_i + 3c_i)\\
&\leq p D_i - \hat{\gamma}_i D_i (p - \frac{w_i}{c_i}) + p (2b_i + 3c_i) \tag{Since $\gamma_i \geq \hat{\gamma}_i$ and $w_i \leq p c_i$}\\
& \leq p (N_i +b_i) - \hat{\gamma}_i N_i (p - \frac{w_i}{c_i}) + p (2b_i + 3c_i) \tag{Since $ N_i  + b_i=  D_i $}\\
&=   p N_i - \hat{\gamma}_i N_i (p - \frac{w_i}{c_i}) + p (3b_i + 3c_i)
\end{align*}

Meanwhile, as $\mathrm{Transport}_i(\textsc{OPT})  = \gamma_i^{\text{OPT}} N_i \frac{w_i}{c_i}$ and  $\mathrm{Penalty}_i(\textsc{OPT})  = (1-\gamma_i^{\text{OPT}}) N_i p$
\begin{equation}
    \cost_i(\textsc{OPT}) = \gamma_i^{\text{OPT}} N_i \frac{w_i}{c_i} + (1-\gamma_i^{\text{OPT}}) N_i p = pN_i -  \gamma_i^{\text{OPT}} N_i(p- \frac{w_i}{c_i})
\end{equation}
Subtracting the 2 terms, we have

\begin{align*}
   \cost_i(\bar{\ell}) -   \cost_i(\textsc{OPT}) &\leq (\gamma_i^{\text{OPT}} - \hat{\gamma}_i )N_i (p- \frac{w_i}{c_i})+p(3b_i + 3c_i)\\
   &\leq |\gamma_i^{\text{OPT}}- \hat{\gamma}_i | N_i (p- \frac{w_i}{c_i})+ p (3b_i + 3c_i) \tag{as $N_i (p- \frac{w_i}{c_i}) \geq 0$}\\
   &\leq |\gamma_i^{\text{OPT}}- \hat{\gamma}_i | N_i p  + p (3b_i + 3c_i) \tag{as $\frac{w_i}{c_i} \leq 0$}\\
\end{align*}
\end{enumerate}

In either cases, as $\tau >0$ and $ |\gamma_i^{\text{OPT}}- \hat{\gamma}_i | N_i p$ is non-negative, we obtain,
\[\cost_i(\bar{\ell})  \leq (1+\tau)\cost_i(\textsc{OPT}) + |\gamma_i^{\text{OPT}}- \hat{\gamma}_i | N_i p  + p (3b_i + 3c_i) \]

Summing across all sites, we have

\begin{align*}
    \cost(\bar{\ell}) &=\sum_{ i \in [n]}\cost_i(\bar{\ell}) \\
    &\leq \sum_{ i\in [n]}(1+\tau)\cost_i(\textsc{OPT}) + |\gamma_i^{\text{OPT}}- \hat{\gamma}_i | N_i p  + p (3b_i + 3c_i)\\
    &=  \sum_{ i\in [n]}(1+\tau)\cost_i(\textsc{OPT}) + \sum_{ i\in [n]} |\gamma_i^{\text{OPT}}- \hat{\gamma}_i | N_i p  + \sum_{ i\in [n]} p (3b_i + 3c_i) \\
    & \leq (1+\tau)\cost(\textsc{OPT}) +3\eta p  + \sum_{ i\in [n]} p (3b_i + 3c_i) \tag{by \cref{lem:smoothness-helper3}}
\end{align*}

Suppose $s^{\mathrm{end}} = 0$. Define
\[
i' = \begin{cases}
n & \text{if $\sum_{i=1}^n \hat{D}_i < \hat{s}$}\\
\hat{i}^\star - 1 & \text{if $\gamma_{\hat{i}^\star} < \min\{1, \frac{\lambda p c_{\hat{i}^\star}}{w_{\hat{i}^\star}} \}$}\\
\hat{i}^\star & \text{otherwise}
\end{cases}
\]
\begin{enumerate}
    \item For $i \in \{1,\dots, i'\}$, we have $\gamma_i \leq \hat{\gamma}_i$ by \cref{lem:smoothness-helper2}. By item 1 of Lemma 4, we have

    \begin{align*}
        \mathrm{transport}_i(\bar{\ell}) &\leq \gamma_iD_i\frac{w_i}{c_i}+ p(b_i+2c_i) \tag{item 2 of \cref{lem:structural}}\\
        &\leq \gamma_i (N_i + b_i)\frac{w_i}{c_i}+ p(b_i+2c_i) \tag{as $N_i + b_i = c_i$}\\
        &\leq  \gamma_i N_i\frac{w_i}{c_i}+ p(2b_i+2c_i) \tag{as $\frac{w_i}{c_i}\leq p$}\\
        &\leq  \hat{\gamma}_i N_i\frac{w_i}{c_i}+ p(2b_i+2c_i) \tag{as $\gamma_i\leq \hat{\gamma}_i$}
    \end{align*}
    Meanwhile, $   \mathrm{transport}_i(\bar{\textsc{OPT}}) = \gamma_i^{\textsc{OPT}}  N_i \frac{w_i}{c_i}$.  Thus, 
    \begin{align*}
        \Delta(\mathrm{transport}_i) &=      \mathrm{transport}_i(\bar{\ell})  - \mathrm{transport}_i(\bar{\textsc{OPT}}) \\
        &\leq (\gamma_i^{\textsc{OPT}}   - \hat{\gamma}_i) N_i \frac{w_i}{c_i} + 2p(b_i+c_i) \\
        &\leq  |\gamma_i^{\textsc{OPT}}   - \hat{\gamma}_i| N_i \frac{w_i}{c_i}  + 2p(b_i+c_i) \tag{as $N_i \frac{w_i}{c_i}) \geq 0$}\\
        &\leq  |\gamma_i^{\textsc{OPT}}   - \hat{\gamma}_i| N_i p  + 2p(b_i+c_i) \tag{as $\frac{w_i}{c_i}\leq p$}\\
    \end{align*}

    \item For $i \in \{i'+1,\dots,n\}$, we have $\gamma_i = \min(1,\frac{\lambda pc_i}{w_i}) \leq \frac{\lambda pc_i}{w_i}$. By item 2 of \cref{lem:structural}, we have 
    $\Delta(\mathrm{transport}_i)
    \leq \lambda \cdot \mathrm{penalty}_i(\OPT) + 2p (b_i + c_i)$.
\end{enumerate}

In either case, as $\lambda >0$ and the term $ |\gamma_i^{\textsc{OPT}}   - \hat{\gamma}_i| N_i p $ is non-negative, we have 

\begin{equation}\label{eq:err}
    \Delta(\mathrm{transport}_i)
    \leq \lambda \cdot \mathrm{penalty}_i(\OPT) +|\gamma_i^{\textsc{OPT}}   - \hat{\gamma}_i| N_i p  + 2p (b_i + c_i)
\end{equation}

Hence, 

\begin{align*}
    \cost(\bar{\ell}) &= \cost(\textsc{OPT}) +  \Delta(\mathrm{transport}) +  \Delta(\mathrm{penalty}) \\
    & = \cost(\textsc{OPT}) +  \Delta(\mathrm{transport}) +  \sum_{i \in [n]}p(b_i+c_i) \tag{by item 3 of \cref{lem:structural}}\\
    &=  \cost(\textsc{OPT}) +  \sum_{i \in [n]} \Delta(\mathrm{transport})_i +  \sum_{i \in [n]} p(b_i+c_i) \\
    &\leq   \cost(\textsc{OPT}) +  \sum_{i \in [n]} \lambda \cdot \mathrm{penalty}_i(\OPT) +|\gamma_i^{\textsc{OPT}}   - \hat{\gamma}_i| N_i p  + 2p (b_i + c_i) +  \sum_{i \in [n]} p(b_i+c_i) \tag{by \cref{eq:err}}\\
    &\leq  \cost(\textsc{OPT}) +  \lambda \cdot \mathrm{penalty}(\OPT) +  \sum_{i \in [n]} |\gamma_i^{\textsc{OPT}}   - \hat{\gamma}_i| N_i p   +  \sum_{i \in [n]} 3p(b_i+c_i) \\
    &\leq (1+\lambda)  \cost(\textsc{OPT}) +  \sum_{i \in [n]} |\gamma_i^{\textsc{OPT}}   - \hat{\gamma}_i| N_i p   +  \sum_{i \in [n]} 3p(b_i+c_i) \tag{as $\mathrm{penalty}(\OPT)  \leq \cost(\textsc{OPT})$}\\
    &\leq (1+\lambda)  \cost(\textsc{OPT}) +  3\eta p   +  \sum_{i \in [n]} 3p(b_i+c_i) \tag{by \cref{lem:smoothness-helper3}}
\end{align*}
\end{proof}

\learningaugmented*
\begin{proof}
Combine the conclusions of \cref{lem:learning-augmented-robustness} and \cref{lem:learning-augmented-smoothness}.
\end{proof}

\paretopoint*
\begin{proof}
Let \(\tau=\lambda\eps\). Since \(0<\eps<1\), we have
\(0<\tau<\lambda\le 1/3\). Define
\(R_\lambda=1+\frac{(1-\lambda)^2}{4\lambda}
=\frac{(1+\lambda)^2}{4\lambda}\) and
\(R_\tau=1+\frac{(1-\tau)^2}{4\tau}
=\frac{(1+\tau)^2}{4\tau}\). Since \(\tau<\lambda\le 1/3\), we have
\(R_\tau>R_\lambda\). Let \(\Delta=R_\tau-R_\lambda>0\).

Suppose, for contradiction, that there exists a possibly randomized online
algorithm \(\ALG\) satisfying both guarantees. Let \(C>0\) be large enough to
dominate the additive terms in both guarantees on the instances constructed
below, so the additive term is at most \(C\sum_{i\in[n]}p(b_i+c_i)\).

\textbf{\texttt{OSSA} instance parameters.}
Let \(n=2\), \(b_1=b_2=c_1=c_2=1\), \(w_1=0\), and
\(w_2=\frac{2\tau}{1+\tau}p\). Then
\(\sum_{i=1}^2 p(b_i+c_i)=4p\). Choose \(K\) large enough so that
\[
K>
\frac{C(1+\tau)}{2\tau\Delta}
\left(4+\frac{2(1-\tau)}{\tau}\right).
\]
The predictions are fixed as \(\hat{s}=K\) and
\(\hat{D}_1=\hat{D}_2=K+1\).

As a reminder, superscripts denote time steps, not powers. For instance,
\(d_i^2\) is the demand arriving at site \(i\) at time step \(2\).

\textbf{Accurate-prediction instance.}
First consider the instance where the prediction is perfect. At time step \(1\),
let \(d_1^1=d_2^1=1\), which consumes the initial stock at both sites. For
\(t=2,\ldots,K+1\), let \(d_1^t=0\) and \(d_2^t=1\). For
\(t=K+2,\ldots,2K+1\), let \(d_1^t=1\) and \(d_2^t=0\). Thus \(s=K+2\) and
\(D_1=D_2=K+1\), so \(\eta=0\).

\textbf{Lower bounding \(\ALG\) on the accurate-prediction instance.}
Let \(X\) be the amount of supply sent by \(\ALG\) to site \(2\) during the block
\(t=2,\ldots,K+1\). Since \(\ALG\) may be randomized, \(X\) is a random variable.

On this instance, \(\OPT\) sends all \(K\) remaining units of supply to site
\(1\). Since \(w_1=0\), this has no transport cost, and \(\OPT\) pays penalty
\(Kp\) for the unmet demand at site \(2\). Hence \(\cost(\OPT)=Kp\).

If \(\ALG\) sends \(X\) units to site \(2\), then it saves \(Xp\) penalty at
site \(2\), but loses exactly \(Xp\) at site \(1\), since those units are no
longer available for site \(1\). These penalty terms cancel. The only additional
cost is transport cost \(Xw_2=X\frac{2\tau}{1+\tau}p\). Therefore,
\[
\cost(\ALG)\ge Kp+X\frac{2\tau}{1+\tau}p.
\]
\paragraph{Upper bounding $\E[x]$ by enforcing consisteny}
By the \(\eta=0\) consistency guarantee,
\[
\E[\cost(\ALG)]\le (1+\tau)Kp+4Cp.
\]
Combining the two inequalities gives
\[
Kp+\E[X]\frac{2\tau}{1+\tau}p\le (1+\tau)Kp+4Cp,
\]
and hence
\[
\E[X]\le \frac{1+\tau}{2}K+\frac{2C(1+\tau)}{\tau}.
\]

\textbf{Inaccurate-prediction instance.}
Now consider a second instance with the same prediction \(\hat{s}=K+2\) and
\(\hat{D}_1=\hat{D}_2=K+1\). The prefix is identical: at time step \(1\),
\(d_1^1=d_2^1=1\), and for \(t=2,\ldots,K+1\), \(d_1^t=0\) and \(d_2^t=1\).
However, there is no later demand at site \(1\). Thus the true demands are
\(D_1=1\) and \(D_2=K+1\), while \(s=K+2\), so the prediction is inaccurate.

Since the two instances have the same prediction and the same prefix up to the
end of the site-\(2\) block, \(\ALG\) has the same distribution over its actions
on this prefix. Therefore the same random variable \(X\) describes the amount
sent to site \(2\) during the first block.

\textbf{Lower bounding \(\ALG\) on the inaccurate-prediction instance.}
On this instance, \(\OPT\) sends all \(K\) remaining units of supply to site
\(2\). Hence \(\cost(\OPT)=Kw_2=\frac{2\tau}{1+\tau}Kp\). By contrast,
\(\ALG\) sends only \(X\) units to site \(2\), leaving at least \(K-X\) units of
site-\(2\) demand unmet. Therefore,
\[
\cost(\ALG)\ge (K-X)p+X\frac{2\tau}{1+\tau}p.
\]
Taking expectations and using the bound on \(\E[X]\), we get
\begin{align*}
\E[\cost(\ALG)]
&\ge Kp-\E[X]\left(1-\frac{2\tau}{1+\tau}\right)p \\
&= Kp-\E[X]\frac{1-\tau}{1+\tau}p \\
&\ge Kp-
\left(\frac{1+\tau}{2}K+\frac{2C(1+\tau)}{\tau}\right)
\frac{1-\tau}{1+\tau}p \\
&= \frac{1+\tau}{2}Kp-\frac{2C(1-\tau)}{\tau}p.
\end{align*}

\textbf{Contradicting robustness.}
The robustness guarantee would require
\(\E[\cost(\ALG)]\le R_\lambda\cost(\OPT)+4Cp\). However,
\begin{align*}
\E[\cost(\ALG)]-R_\lambda\cost(\OPT)
&\ge
\left(
\frac{1+\tau}{2}
-
R_\lambda\frac{2\tau}{1+\tau}
\right)Kp
-
\frac{2C(1-\tau)}{\tau}p \\
&=
\frac{2\tau}{1+\tau}(R_\tau-R_\lambda)Kp
-
\frac{2C(1-\tau)}{\tau}p \\
&=
\frac{2\tau\Delta}{1+\tau}Kp
-
\frac{2C(1-\tau)}{\tau}p.
\end{align*}
By the choice of \(K\),
\[
\frac{2\tau\Delta}{1+\tau}K>
4C+\frac{2C(1-\tau)}{\tau}.
\]
Therefore \(\E[\cost(\ALG)]-R_\lambda\cost(\OPT)>4Cp\), equivalently
\(\E[\cost(\ALG)]>R_\lambda\cost(\OPT)+4Cp\), contradicting robustness.

Hence no possibly randomized online algorithm can simultaneously satisfy both
guarantees.
\end{proof}
\section{Experimental details}
\label{sec:appendix-experiments}

Additional experimental results are given in \cref{fig:synthetic-experiments} and \cref{fig:taxi-experiments}.
All experiments are run locally on a 2024 MacBook Pro with M4 chip; no GPU compute is required.
Each set of experiments takes at most an hour to run.
Our source code is available on Github.\footnote{\url{https://github.com/cxjdavin/online-allocation-with-unknown-shared-supply}}

\subsection{Policies compared, evaluation metrics, and qualitative takeaways}

\paragraph{Policies compared.}

\begin{enumerate}
    \item \textbf{\OPT}: The optimal offline algorithm described in \cref{sec:online}. 
    
    \item \textbf{\GPA}: Our $\bar{\gamma}\mathrm{PA}$ algorithm from \cref{alg:proportional-allocation}.
    
    \item \textbf{AlwaysFill}: An aggressive replenishment policy. After each round of demand, this policy always requests enough supply to restore each site's inventory to at least $b_i$, rounded up to $c_i$ to ``extract'' maximum utility from each fixed transportation cost $w_i$ incurred.
    
    \item \textbf{$\rho$-Greedy}: This policy is additionally given the actual $\rho = s / \sum_i D_i$ as input, something that an online algorithm should \emph{not} have.
    Given $\rho$, this policy attempts to maintain cumulative allocation proportional to $\rho$ times the cumulative observed demand. This mimics $R^t$ from \GPA{} without the $r_i^t < b_i$ trigger since it knows $\rho$ upfront.
    
    \item \textbf{$\rho$-CoinFlip}: This policy is additionally given the actual $\rho = s / \sum_i D_i$ as input, something that an online algorithm should \emph{not} have.
    This policy is another variant on how one might try to exploit the knowledge of $\rho$.
    For each site $i \in [n]$, this policy independently resupplies the site up to $b_i$ with probability $\rho$.
    
    \item \textbf{Backlog}: This policy is inspired by backlogging strategies from the \texttt{OWMR} and \texttt{JRP} literature. For each site $i \in [n]$, this policy tracks accumulated unmet demand since the last successful resupply and requests enough supply to restore the inventory to at least $b_i$ once the accumulated unmet demand exceeds $w_i$.
    
    \item \textbf{$\LAGPA(\lambda)$}: We run \cref{alg:proportional-allocation} using thresholds $\bar{\gamma}$ produced by \cref{alg:prediction} when given distrust parameter $\lambda \in (0, 1/3]$ and the predictions $\hat{s}$ and $\{ \hat{D}_i \}_{i \in [n]}$. For each $\lambda$ value tested, we provide the policy either with perfect advice ($\eta = 0$) or bad advice ($\eta = 10s$). The error in the advice is generated by adding scaled random noise such that that $\eta \approx 10s$.
\end{enumerate}

\paragraph{Evaluation metrics.}
For each set of experiments, we record the cost incurred (\cref{eq:main-objective}) as well as compare the ratios between $\cost(\ALG) / \cost(\OPT)$ for each policy.
For visual clarity, we separate the ratio comparisons into two sets of plots.
The first set compares \GPA{} with other baselines, while the second set compares \GPA{} with $\LAGPA(\lambda)$.

\paragraph{Qualitative takeaways.}

Our experimental results are given in \cref{fig:synthetic-experiments} and \cref{fig:taxi-experiments}.
On the X-axis, we measure the fraction of demand ($\rho$) that is available as supply.
On the Y-axis, we measure $\cost(\ALG)$ and $\cost(\ALG)/\cost(\OPT)$ as described earlier.
Across all settings, we see that \GPA{} outperforms all other methods, especially when the global supply is scarce: it consistently has a lower $\cost(\ALG)/\cost(\OPT)$ ratio than the others in the middle plot.
Further details on how the \OSSA{} instances are generated are given in \cref{sec:appendix-synthetic} and \cref{sec:appendix-taxi}.

Let us now interpret and discuss the plots in further detail.
On the left plot, we see that $\cost(\OPT)$ decreases at a decreasing rate.
This is because $w_i \leq p c_i$ and \OPT{} tries to satisfy demands at sites with lower $w_i / c_i$ first.
By construction, the AlwaysFill, $\rho$-Greedy, and $\rho$-CoinFlip baselines match $\OPT$ when $\rho = 1$ so their ratio drops to $1$ for large $\rho \geq 1$.
Meanwhile, \GPA{} decreases roughly linearly and then plateaus into a horizontal line before we hit $\rho = 1$, i.e., it does not fully exploit the all the available supply even if it is available.
This is because of the $\gamma_i$ threshold parameter which tries ensure that it does robust rationing in the face of supply uncertainty.
This inflection of $\cost(\GPA)$ also explains the dip and eventual rising of $\cost(\GPA)/\cost(\OPT)$ in the middle plot.
Finally, there are two trends to observe in the rightmost plot.
First, we see that the curves $\LAGPA(\lambda)$ increases as $\lambda$ increases for $\eta = 0$ (bluish curves), and increases as $\lambda$ increases for $\eta = 10s$ (reddish curves).
Second, these curves diverge from the original \GPA{} curves as $\lambda$ transitions from $1/3$ towards $0$.
These trends empirically validate \cref{thm:learning-augmented}.

\begin{figure}[htb]
    \centering
    \includegraphics[width=\linewidth]{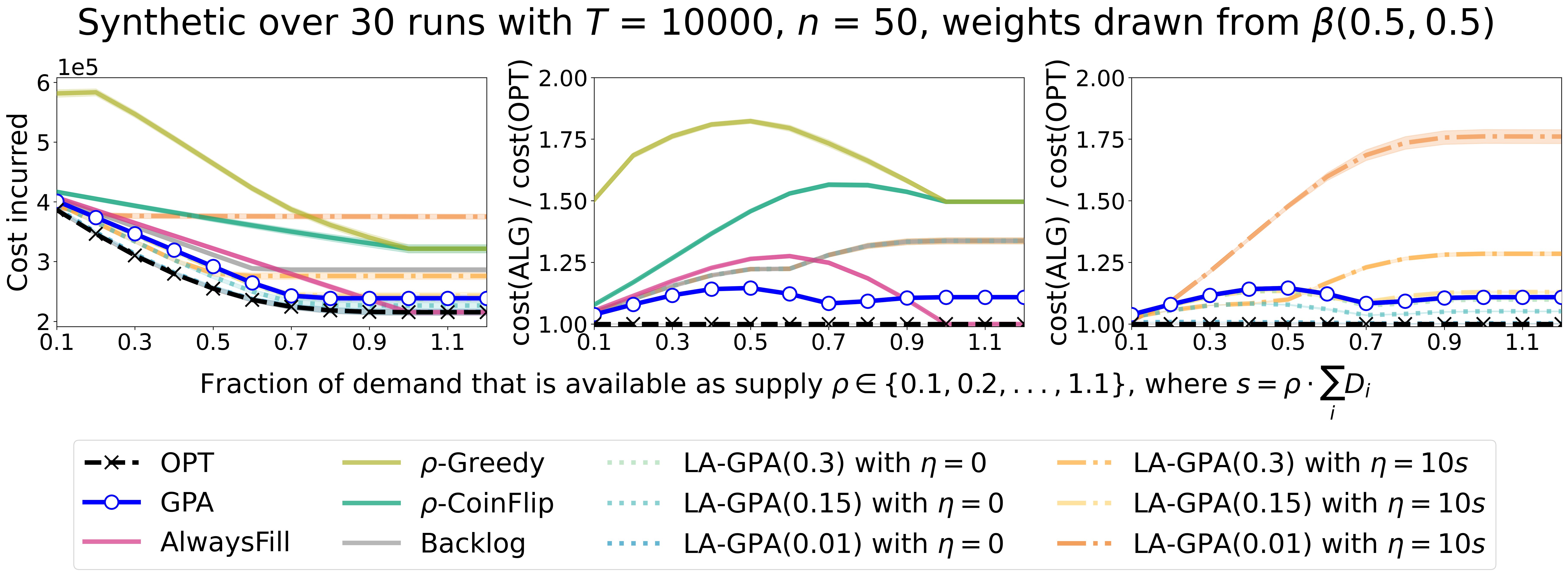}
    \includegraphics[width=\linewidth]{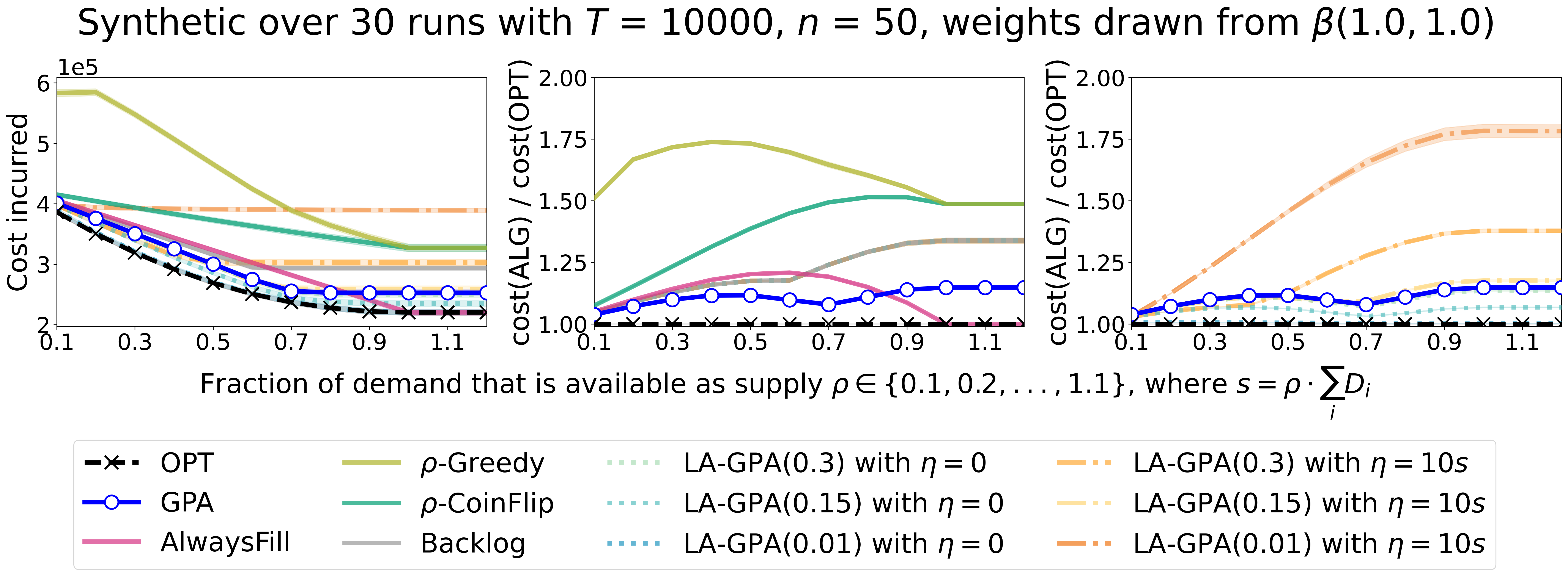}
    \includegraphics[width=\linewidth]{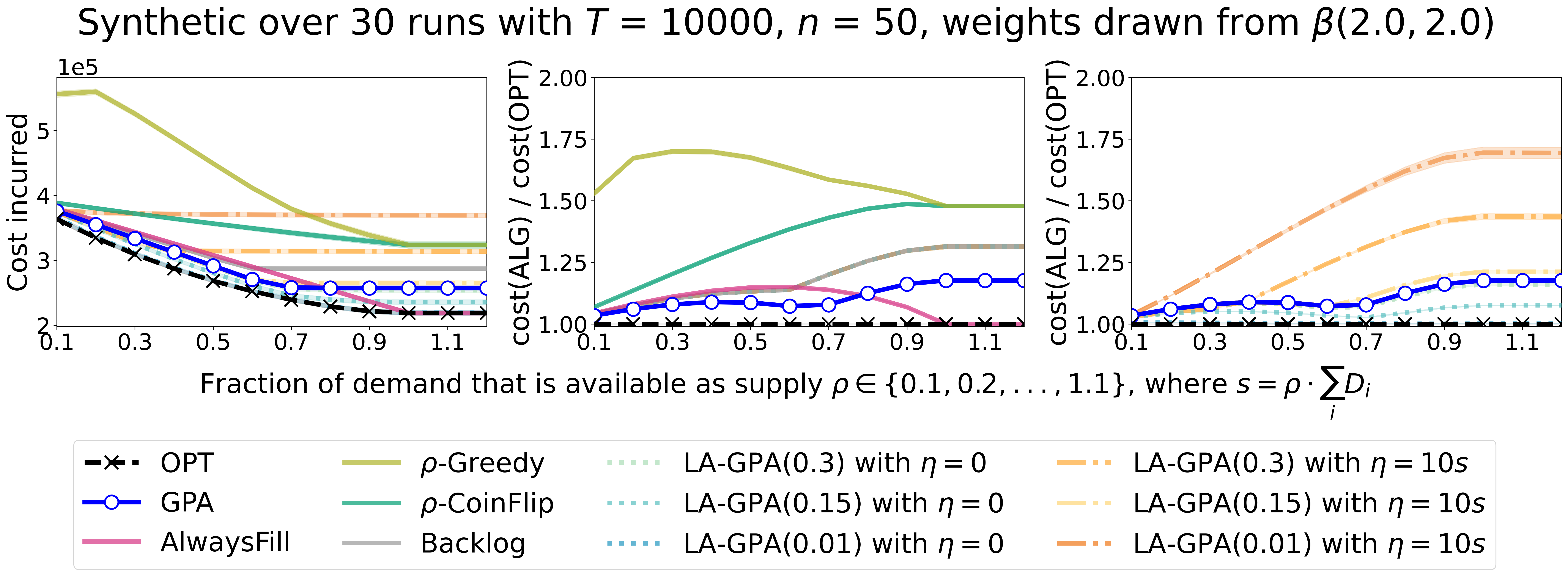}
    \caption{Synthetic experiments over $30$ runs on an \OSSA{} instance with $n = 50$ sites and $T = 10,000$ time steps. To simulate different distributions of weights, each set of experiment uses weights drawn independently from $\beta(x,x)$ distribution for different $x \in \{0.5, 1.0, 2.0\}$ parameters. See \cref{sec:appendix-synthetic} for further details.}
    \label{fig:synthetic-experiments}
\end{figure}

\begin{figure}[htb]
    \centering
    \includegraphics[width=\linewidth]{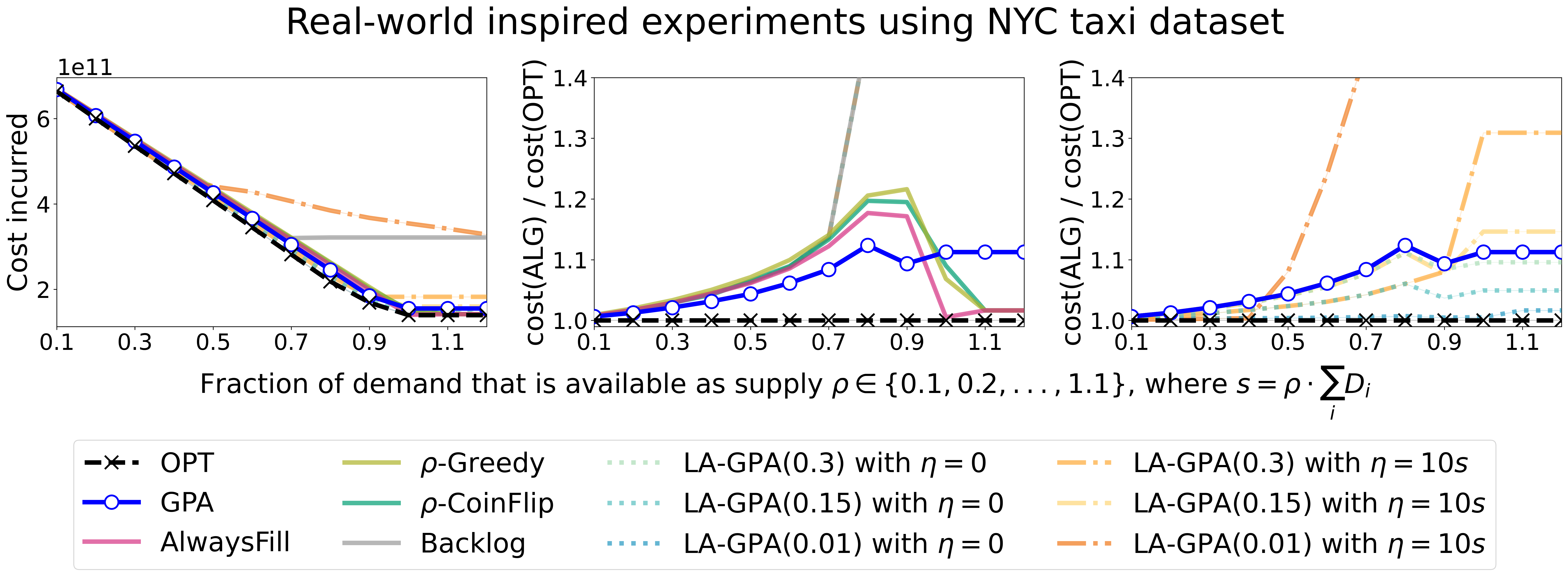}
    \caption{Real-world inspired experiment by repurposing the NYC taxi dataset See \cref{sec:appendix-taxi} for further details.}
    \label{fig:taxi-experiments}
\end{figure}

\subsection{Further details for synthetic experiments}
\label{sec:appendix-synthetic}

We first evaluate the policies on synthetic \OSSA{} instances. Each instance consists of $n=50$ sites over a horizon of $T=10000$ time periods. For each site $i$, we sample the transportation weight $w_i \sim \mathrm{Beta}(\alpha,\beta)$ and set the shipment capacity $c_i=10$. We sample the one-step demand bound $b_i \sim \mathrm{Poisson}(10)$, and generate demand by
\[
d_i^t=\min\{b_i,X_i^t\}, \qquad X_i^t\sim \mathrm{Poisson}(b_i).
\]
Thus, the demand process is random and site-dependent while satisfying the model requirement $d_i^t\le b_i$.

We set
\[
p=10^{-6}+\max_i \frac{w_i}{c_i},
\]
which ensures $w_i\le pc_i$ for all sites. We vary the supply level over $\rho\in\{0.1,0.2,\ldots,1.2\}$ by setting
\[
s=\left\lfloor \rho\left(\sum_iD_i-\sum_i b_i\right)\right\rfloor,
\]
where $D_i=\sum_{t=1}^T d_i^t$ is the total demand at site $i$. We repeat each setting over $30$ independent runs and test three transportation-weight distributions, $(\alpha,\beta)\in\{(0.5,0.5),(1,1),(2,2)\}$. 

\subsection{Further details for real-world inspired experiments}
\label{sec:appendix-taxi}

For the real-world dataset, we use New York City yellow taxi trip records and taxi-zone geographic data from the New York City Taxi and Limousine Commission (TLC) to simulate spatially distributed demand over time and construct site-dependent transportation weights $w_i$ for each site $i$ \cite{nyc_tlc_trip_data}. Specifically, we combine yellow taxi trip records from January to March 2026 and aggregate pickups into daily demand over $T=90$ time periods. After filtering to this period, the dataset contains $11{,}077{,}196$ pickup records across $262$ observed pickup taxi zones.

We exploit the geographic information in the NYC yellow taxi dataset, including the TLC taxi-zone shapefile and trip records. Each trip record reports a pickup taxi-zone identifier rather than an exact pickup coordinate, so we treat the TLC taxi zone as the finest observed geographic unit for each pickup. Each taxi zone is represented by the centroid of its polygon in a projected NYC coordinate system. We then group taxi zones into local sites using the borough and TLC service-zone labels. In the raw TLC metadata, borough labels include Manhattan, Queens, Brooklyn, Bronx, Staten Island, and EWR, while service-zone labels include Yellow Zone, Boro Zone, Airports, and EWR. Combining borough and service-zone information gives a small number of geographically meaningful local regions, which serve as our local \OSSA{} sites, such as \texttt{Queens---Airports}.

For each local site, we compute a demand-weighted centroid of its constituent taxi-zone centroids. Specifically, suppose site $i$ contains taxi zones $j \in S_i$, where $S_i$ is the set of taxi zones assigned to site $i$. Let $q_j$ be the total number of pickups in taxi zone $j$ over the three-month period, and let $(x_j,y_j)$ be the centroid of taxi zone $j$. Then the site centroid is
\[
\bar{x}_i = \frac{\sum_{j \in S_i} q_j x_j}{\sum_{j \in S_i} q_j},
\qquad
\bar{y}_i = \frac{\sum_{j \in S_i} q_j y_j}{\sum_{j \in S_i} q_j}.
\]
Since Manhattan Yellow Zone accounts for a large fraction of total pickups, we further split it into northern and southern subregions using the demand-weighted median of taxi-zone centroid $y$-coordinates, producing \texttt{Manhattan---Yellow Zone South} and \texttt{Manhattan---Yellow Zone North}. This produces $n=9$ local sites with centroid coordinates, which we use as proxies for local clinics in the \OSSA{} instance. We then place the central warehouse at the demand-weighted centroid of all local sites and define $w_i$ as the projected Euclidean distance from the warehouse centroid to the centroid of site $i$. The number of daily pickups assigned to site $i$ is used as the demand $d_i^t$ at each time period $t \in [T]$.

This dataset construction is meaningful because it produces a large-scale, geographically structured demand sequence with transportation weights $w_i$ derived from real NYC taxi-zone geometry. In this real-world instance, \GPA{} achieves strong performance across different supply levels $\rho$. In our implementation, the total supply is set as
$s = \left\lfloor \rho\left(\sum_i D_i - \sum_i b_i\right)\right\rfloor,$
where $D_i$ is the total demand at site $i$ and $b_i$ is the maximum daily demand at site $i$. The results show that \GPA{} outperforms strong supply-aware baselines across most supply levels, demonstrating the effectiveness of our algorithm on spatially imbalanced real-world demand. Furthermore, we evaluate the learning-augmented algorithm across different distrust hyperparameters $\lambda \in (0,1/3]$ and advice qualities $\eta \ge 0$.





\end{document}